\documentclass{article}

\usepackage{PRIMEarxiv}

\usepackage[utf8]{inputenc} 
\usepackage[T1]{fontenc}    
\usepackage{hyperref}       
\usepackage{url}            
\usepackage{booktabs}       
\usepackage{amsfonts}       
\usepackage{nicefrac}       
\usepackage{microtype}      
\usepackage{lipsum}
\usepackage{fancyhdr}       
\usepackage{graphicx}       
\graphicspath{{media/}}     

\usepackage{soul}

\usepackage{amsmath}
\usepackage{amssymb}
\usepackage{mathtools}
\usepackage{amsthm}
\usepackage{bm}
\usepackage{mathtools}
\usepackage{thm-restate}
\usepackage{multirow}  
\usepackage[table,dvipsnames]{xcolor}
\usepackage{subcaption}

\usepackage[framemethod=TikZ]{mdframed}
\mdfdefinestyle{contribution}{%
    middlelinecolor =   none,
    middlelinewidth =   1pt,
    backgroundcolor =   blue!10,
    roundcorner     =   5pt,
}
\mdfdefinestyle{takeaway}{%
    middlelinecolor =   none,
    middlelinewidth =   1pt,
    backgroundcolor =   teal!10,
    roundcorner     =   5pt,
}
\usepackage{soul}

\usepackage{hyperref}
\hypersetup{
    colorlinks=true,
    linkcolor=orange,
    filecolor=magenta,      
    urlcolor=cyan,
    citecolor=blue,
    pdftitle={eri},
    pdfpagemode=FullScreen,
    }

\usepackage[capitalize,noabbrev]{cleveref}   
\usepackage{lipsum}

\newcommand{\F}{\mathcal{F}}
\newcommand{\Wre}{W_{\mathrm{re}}}
\newcommand{\Wim}{W_{\mathrm{im}}}
\newcommand{\x}{\textbf{x}}

\newcommand{\beps}{\bm{\epsilon}}
\newcommand{\boldf}{\bm{f}}
\newcommand{\boldw}{\bm{w}}
\newcommand{\boldG}{\bm{G}}
\newcommand{\boldA}{\Lambda}

\newcommand{\R}{\mathbb{R}}
\newcommand{\C}{\mathbb{C}}
\newcommand{\Cconstr}{\C_{\mathrm{constr}}}
\newcommand{\N}{\mathbb{N}}
\newcommand{\boldb}{\bm{b}}
\newcommand{\boldh}{{\varphi^{-1}}}
\renewcommand{\d}{\mathrm{d}}
\newcommand{\s}{\mathbf{s}}
\newcommand{\takeaway}[1]{\textbf{\textcolor{teal}{Take-away #1.}}}
\newcommand{\boldv}{\bm{v}}
\newcommand{\y}{\textbf{y}}
\newcommand{\trunc}{\varphi}

\newcommand{\scorematching}{\mathcal{L}_{\mathrm{SM}}}

\newcommand{\bmag}[1]{\textbf{\textcolor{blue!70}{#1}}}
\newcommand{\D}[1]{\mathcal{D}_{\mathrm{#1}}}
\renewcommand{\S}[1]{\mathcal{S}_{\mathrm{#1}}}
\newcommand{\tildeS}[1]{\tilde{\mathcal{S}}_{\mathrm{#1}}}
\newcommand{\tildeD}[1]{\tilde{\mathcal{D}}_{\mathrm{#1}}}
\newcommand{\nyquist}{\omega_{\mathrm{Nyq}}}
\newcommand{\indep}{\perp \!\!\! \perp}
\newcommand{\deloc}[1]{\Delta_\mathrm{#1}}
\newcommand{\z}{\mathbf{z}}

\theoremstyle{plain}
\newtheorem{theorem}{Theorem}[section]
\newtheorem{proposition}[theorem]{Proposition}
\newtheorem{lemma}[theorem]{Lemma}

\theoremstyle{definition}

\theoremstyle{remark}
\newtheorem{remark}[theorem]{Remark}

\definecolor{Gray}{gray}{0.95}
\newcolumntype{g}{>{\columncolor{Gray}}r}
\newcolumntype{e}{>{\columncolor{Gray}}l}

\title{Time Series Diffusion in the Frequency Domain
}

\author{
  Jonathan Crabbé\thanks{Denotes equal contribution and corresponding authors.}, \ \ Nicolas Huynh\footnotemark[1], \ \ Jan Stanczuk, \ \ Mihaela van der Schaar \\
  DAMTP \\
  University of Cambridge \\
  \texttt{\{jc2133, nvth2, js2164, mv472\}@cam.ac.uk} \\
}

\begin{document}
\maketitle

\begin{abstract}
Fourier analysis has been an instrumental tool in the development of signal processing. This leads us to wonder whether this framework could similarly benefit generative modelling.  In this paper, we explore this question through the scope of time series diffusion models. More specifically, we analyze whether representing time series in the frequency domain is a useful inductive bias for score-based diffusion models. By starting from the canonical SDE formulation of diffusion in the time domain, we show that a dual diffusion process occurs in the frequency domain with an important nuance: Brownian motions are replaced by what we call mirrored Brownian motions, characterized by mirror symmetries among their components. Building on this insight, we show how to adapt the denoising score matching approach to implement diffusion models in the frequency domain. This results in frequency diffusion models, which we compare to canonical time diffusion models. Our empirical evaluation on real-world datasets, covering various domains like healthcare and finance, shows that frequency diffusion models better capture the training distribution than time diffusion models. We explain this observation by showing that time series from these datasets tend to be more localized in the frequency domain than in the time domain, which makes them easier to model in the former case. All our observations point towards impactful synergies between Fourier analysis and diffusion models.  
\end{abstract}

\section{Introduction} \label{sec:introduction}




Deep generative modelling leverages the inductive bias of neural networks to learn complex, high-dimensional probability distributions from real-world datasets. Among other applications, generative models allow for generation of new synthetic samples consistent with the distribution of the training data, yet distinct from the actual data encountered during training. Recently this field has seen tremendous progress in various modalities including image \cite{karras2020stylegan2, dhariwal2021diffusion_beat_GANs}, audio \cite{kong2021diffwave, donahue2018adversarial_audio}, video \cite{rombach2022video} and text~\cite{dieleman2022continuous} generation, as well as addressing inverse problems such as in-painting \cite{lugmayr2022repaint} or super-resolution \cite{saharia2022sr3}. Moreover deep generative models have started showing significant potential in contributing to natural sciences, though protein design \cite{watson2023rfdiffusion}, drug development \cite{xu2022geodiff} and material synthesis \cite{zeni2023mattergen}. However, the application of these models to time series data has not seen the same level of advancement \cite{gatta2022nn_time_series}. Some notable examples of time series generative models include TimeGAN \cite{yoon2019timegan}, FourierFlow \cite{alaa2021fourierflows}, and RCGAN \cite{esteban2017rcgan}, yet this area remains less explored compared to other applications.

\textbf{Diffusion Models.} In recent years, diffusion models \cite{hyvarinen2005estimation, sohl2015diffusion, ho2020ddpm, song2020score} have emerged as one of the most promising research avenues in deep generative modelling, achieving state-of-the art results across many generative modelling tasks \cite{dhariwal2021diffusion_beat_GANs, saharia2022sr3}. Diffusion models have been applied to time series modelling, achieving promising results \cite{lin2023times_series_diffusion_survey}. However, there is substantial room for development and refinement in these early-stage applications. 

\textbf{Fourier analysis.} Fourier analysis is a remarkably powerful tool in signal processing, compression and machine learning \cite{korner2022fourier}. It has been shown to significantly improve state-of-the-art the performance of many deep learning based time series analysis techniques  \cite{yi2023deep_learning_ts_fourier_survey}, with some recent applications in dataset distillation~\cite{shin2023frequency}. In the context of deep generative models, this is exemplified in \cite{alaa2021fourierflows}, where the application of normalizing flows to Fourier representations yielded promising results. More recently, some work by \cite{phillips2022spectral} has been done on diffusion on functional spaces, which include Fourier representations of signals although the paper does not specifically focus on the Fourier basis.

\textbf{Motivation.} Despite Fourier analysis' widespread success, its application to diffusion models for time series remains largely unexplored. This paper seeks to fill this research gap, by examining whether spectral representations can improve diffusion models for time series modelling. Our focus is not on achieving state-of-the-art results, but rather investigating whether representing time series in the
frequency domain is a useful inductive bias for
diffusion models.

\begin{mdframed}[style=contribution]
\bmag{Our contributions.} 
\bmag{(1)~Formalizing frequency diffusion.}
In \cref{sec:theory}, we show theoretically how to translate SDE-based diffusion of time series to the frequency domain. We demonstrate that the denoising score matching recipe can be adapted by replacing standard Brownian motions by what we call mirrored Brownian motions, characterized by mirror symmetries in their components.
\bmag{(2)~Comparing time and frequency diffusion.}
In \cref{subsec:sample_quality},  we compare the ability of the time and frequency score models to generate samples that are faithful to the training sets by leveraging \emph{sliced Wasserstein distances}. Through an extensive analysis on 6 real-world datasets illustrating fields like healthcare, finance, engineering and climate modelling, we demonstrate that frequency score models consistently outperform the time score models.
\bmag{(3)~Understanding why and when frequency diffusion is preferable.} 
In \cref{subsec:sample_differences}, we demonstrate that the signals in all 6 datasets concentrate most of their power spectrum on the low frequencies. We hypothesize that this localization in the frequency domain explains the superior performances of frequency diffusion models.  In \cref{subsec:when_frequency}, we confirm this hypothesis by artificially delocalizing the spectral representation of real signals and showing that the gap between time and frequency diffusion closes. 
\end{mdframed}

\section{Background} \label{sec:background}
\textbf{Notations.} We consider multivariate time series of fixed size\footnote{Padding over time can be used in cases where the datasets contain time series of different lengths.} $\x \in \R^{N \times M}$, where $N \in \N$ is the number of time steps and $M \in \N$ is the number of features tracked over time. Often, we will denote by $d_X = N \cdot M$ the total dimension of the time series $\x$. We shall use \emph{Greek letters} for components of the time series. In this way, $\x_\tau \in \R^{M}$ denotes the feature vector at time $\tau \in [N]$ and $x_{\tau , \nu}$ denotes the value of feature $\nu \in [M]$ at time $\tau$. We denote by  $[K] := \{ 0, 1 , \dots , K-1 \}$ the integers between $0$ (included) and $K \in \N$ (excluded). To avoid any confusion between time series steps and diffusion steps, we shall use \emph{Latin letters} for the diffusion process. In this way, the diffusion process is described by a family of time series $\{ \x(t) \in \R^{d_X} \}_{t=0}^T$ indexed by a continuous diffusion variable $t \in [0, T]$. Thanks to these notations, we unambiguously interpret $\x_\tau(t) \in \R^{M}$ as the feature vector at time step $\tau \in [N]$ and at diffusion step $t \in [0,T]$. We shall detail below how this diffusion process is defined.

\subsection{Score-based generative modeling with SDEs}
In continuous-time diffusion modelling, one assumes access to samples drawn from an unknown density   $p_{\mathrm{data}}$. The objective of generative modelling is to obtain a tractable approximation of this distribution.

\textbf{Forward diffusion.} Score-based generative modeling with \emph{stochastic differential equation} (SDEs) \cite{song2020score} typically operates by first constructing a forward diffusion process.  In the case of time series, forward continuous diffusion is described by the following SDE, with $t \in [0, T]$: 
\begin{equation} \label{eq:forward_process}
    \d \x = \boldf(\x,t)\d t + \boldG(t) \d\boldw,
\end{equation}
where $\boldf: \mathbb{R}^{d_X} \times [0,T] \rightarrow \mathbb{R}^{d_X}$ is the drift, $\boldw$ is a standard Brownian motion in $\mathbb{R}^{d_X}$, and $\boldG: [0, T] \rightarrow \mathbb{R}^{N \times N}$ is the diffusion matrix. We denote $p_t$ the probability density of the solution $\x(t)$ of \cref{eq:forward_process} at time $t \in [0, T]$. With the slight abuse of notation from \cite{song2020score}, we shall abbreviate $p_t(\x(t))$ by $p_t(\x)$. Together with the SDE, we impose the initial condition $p_{0} = p_{\mathrm{data}}$, which corresponds to samples initially drawn from the data density  $p_{\mathrm{data}}$. In practice, we consider $\boldf$ and $\boldG$ such that $p_{\mathrm{data}}$ is transported to a final density $p_{T}$ close to an isotropic Gaussian.

\textbf{Reverse diffusion.} The reverse diffusion process performs the inverse transformation by transporting the isotropic Gaussian density $p_T$ to the data density $p_0 = p_{\mathrm{data}}$. Hence, applying reverse diffusion to samples drawn from the isotropic Gaussian permits to sample from the unknown density $p_{\mathrm{data}}$. It was shown by \cite{anderson1982} that this reverse diffusion satisfies the following SDE:
\begin{equation} \label{eq:backward_process}
    \d\x = \boldb(\x, t)\d t + \boldG(t)\d\hat{\boldw},
\end{equation}
where $\boldb(\x,t) =  \boldf(\x,t) - \boldG(t) \boldG(t)^{T} \nabla_{\x}\log p_{t}(\x)$, $\d t$ is a negative infinitesimal time step, and  $\hat{\boldw}$ is a Brownian time increment with time going backwards from $T$ to $0$.

\textbf{Denoising score matching.} In order to run the reverse diffusion process, one needs access to the score $\s(\x, t) := \nabla_\x \log p_t(\x)$. In practice, the ground-truth density $p_t$ is unknown. Denoising score matching circumvents this problem by estimating the ground-truth score with a function $\s_{\theta^*}$ whose parameters $\theta^*$ \emph{minimize} the following score matching objective computed from the data samples~\cite{hyvarinen2005estimation, song2019generative}:
\begin{align} \label{eq:score_matching}
   &\theta^* = \arg \min_{\theta \in \Theta}  \mathbb{E}_{t, \x(0), \x(t)} \left[ \scorematching(\s_\theta, \s_{t \vert 0}, \x, t) \right] \\
   &\scorematching(\s_\theta, \s_{t \vert 0}, \x, t) := \|\s_{\theta}(\x,t) - \s_{t \vert 0}(\x, t)
     \|^2 
\end{align}
where $\| \cdot \|$ denotes the Frobenius norm, $t \sim \mathcal{U}(0,T)$, $\x(0) \sim p_0(\x)$, $\x(t) \sim p_{t \mid 0}(\x(t) \vert \x(0))$ with $p_{t\vert 0}$ denoting the transition kernel from $0$ to $t$, and $\s_{t \vert 0}(\x, t) := \nabla_{\x(t)}\log p_{t \vert 0}(\x(t) \vert \x(0))$. With sufficient model capacity, the parameters $\theta^*$ provide an approximation $\s_{\theta^*}$ that is equal to the score $\s$ for almost all $\x$ and $t$ in the large data limit~\cite{vincent2011connection}. Equipped with an approximation of the score $\s_{\theta^*} \approx \s$, one can generate data by sampling according to the solution defined by the reverse diffusion process from \cref{eq:backward_process}.

\subsection{Discrete Fourier Transform}

\textbf{DFT.} By considering a time series $\x= (\x_{0}, \dots , \x_{N-1}) \in \R^{d_X}$, the \emph{Discrete Fourier Transform} (DFT), denoted as $ \tilde{\x} = \F [\x]$, is defined as 
\begin{equation} 
\tilde{\x}_{\kappa} := \frac{1}{\sqrt{N}}\sum_{\tau = 0}^{N-1} \x_{\tau} \exp \left(-\frac{\kappa  2\pi i}{N}  \tau \right)
\end{equation} 
for all $\kappa \in [N]$. In the signal processing literature, each $\kappa$ 
 corresponds to a harmonic of frequency $\omega_{\kappa} := \frac{\kappa 2 \pi}{ N}$. For this reason, the DFT $\tilde{\x}$ is said to represent the time series $\x$ in the \emph{frequency domain}, as opposed to the \emph{time domain}. We also note that the DFT is complex-valued ($\tilde{\x} \in \C^{d_X}$). 

 \textbf{Matrix representation.} We note that the DFT operator $\F$ is \emph{linear} with respect to the time components $(\x_{0}, \dots , \x_{N-1})$. It can therefore be expressed through a left matrix multiplication $\tilde{\x} = \F[\x] = U \x$, where $U \in \mathbb{C}^{N\times N}$ is defined as $[U]_{\kappa \tau} := N^{- 1 / 2} \exp (- i \omega_{\kappa} \tau )$. It can easily be checked (see \cref{subapp:theory_dft}) that the matrix $U$ is unitary: $U^* U = U U^* = I_N$, where $U^*$ is the conjugate transpose of $U$ and $I_N$ is the $N \times N$ identity matrix. This implies  that the DFT operator is invertible and that the original time series can be reconstructed from its representation in the frequency domain: $\x = \F^{-1}[\tilde{\x}] := U^* \tilde{\x}$.

 \textbf{DFT of a real-valued sequence.} 
While the DFT $\tilde{\x}$ is defined in $\mathbb{C}^{d_X}$, some of its components are made redundant by the fact that $\x$ is a real-valued time-series. One can easily check (see \cref{subapp:theory_dft}) that this constraint imposes the following \emph{mirror symmetry} on the DFT for all $\kappa \in [N]$: 
\begin{equation} \label{eq:real_constraint}
    \tilde{\x}_{\kappa} = \tilde{\x}_{N-\kappa}^*,
\end{equation}
where $z^*$ denotes the complex conjugate of $z \in \C$ and we define $\tilde{\x}_{N} := \tilde{\x}_{0}$ for consistency. Through this symmetry, we observe that the components $ \tilde{\x}_\kappa$ with $\kappa \leq  \lfloor N / 2 \rfloor$ uniquely define the DFT of a real-valued time series. For this reason, the frequencies beyond the Nyquist frequency $\nyquist := \omega_{\lfloor N/2 \rfloor}$ are redundant with respect to the lower frequencies. In the frequency domain, one then needs only to diffuse $N$ real numbers extracted from the DFT and the rest of $\tilde{\x}$ can be deduced from \cref{eq:real_constraint}. \\

\textbf{Signal Energy.} An important quantity related to a time series $\x$ is its \emph{total energy}, which simply corresponds to the squared Frobenius norm $\| \x \|^2 := \sum_{\tau=0}^{N-1} \sum_{\nu=1}^M |x_{\tau, \nu}|^2$, where $|\cdot|$ denotes the modulus of a complex number. Through \emph{Parseval's theorem}, this energy can be evaluated by computing the same norm for the DFT $\tilde{\x}$ of $\x$: $\| \x \|^2 = \| \tilde{\x} \|^2$. We note that the total energy is obtained by summing over all time steps or frequencies. To characterize how the energy is distributed over the time steps $\tau \in [N]$, we use the \emph{energy density} defined as the squared Euclidean norm $\| \x_{\tau} \|^2_2 := \sum_{\nu=1}^M |x_{\tau, \nu}|^2$. Similarly, the \emph{spectral energy density} defined as $\| \tilde{\x}_\kappa \|_2^2$ describes how the signal energy is distributed across the frequencies $\kappa \in [N]$.

\textbf{Probability density in the complex space.} Adapting the diffusion formalism to the frequency domain requires to define a probability density for the complex-valued random variable $\tilde{\x} \in \C^{d_X}$. By following \cite{schreier2010}, this density is written in terms of the real and imaginary parts of the signal $\tilde{p}(\tilde{\x}) := \tilde{p}(\Re [\tilde{\x}], \Im[ \tilde{\x} ] )$.  Similarly, the score function follows a similar decomposition in terms of the signal real and imaginary parts $\tilde{\s}(\tilde{\x}) := \nabla_{\Re [\tilde{\x}]} \log \tilde{p}(\tilde{\x}) + i \cdot \nabla_{\Im [\tilde{\x}]} \log \tilde{p}(\tilde{\x})$. We note that the gradient involved in the definition of the scores is non-trivial when the constraint in \cref{eq:real_constraint} is enforced. In \cref{subapp:theory_submanifold}, we establish a formal definition in this setting by interpreting the complex signals fulfilling this constraint as a submanifold in $\C^{d_X}$. This constraint implies that the score components follow an analogous mirror symmetry: $\tilde{\s}_{\kappa} = \tilde{\s}_{N - \kappa}^*$ for all $\kappa \in [N]$. In the following, we shall implicitly rely on this definition.

\section{Diffusing in the frequency domain} \label{sec:theory}
In the previous section, we have described how the typical diffusion formalism applies to time-series. We have also described how the DFT $\tilde{\x} = \F[\x]$ offers a full description of the time series $\x$ in the frequency domain. The first step is to define how time-based diffusion translates in the frequency domain. Note that this is non-trivial as the DFT are complex-valued $\tilde{\x} \in \C^{d_X}$ signals. To solve this, we shall assume that the stochastic process in the time domain $\{\x(t) \}_{t=0}^{T}$, written compactly as $\x$, follows the diffusion process described in \cref{eq:forward_process}. By leveraging the matrix formulation of the DFT $\tilde{\x} = U \x$, we will now derive diffusion SDEs in the frequency domain.

\subsection{Diffusion SDEs}

In order to derive the diffusion SDEs in the frequency domain, we shall simply apply the DFT operator to the forward diffusion SDE in the time domain from \cref{eq:forward_process}. We note that this equation contains a standard Brownian motion $\boldw$. In the below lemma, we describe the DFT of $\boldw$ and show that it contains two copies of a non-standard Brownian motion related by the constraint from \cref{eq:real_constraint}. We refer to this as a \emph{mirrored Brownian motion}. 

\begin{restatable}{lemma}{dftbrownianlemma}(DFT of standard Brownian motion).\label{lemma:brownian_transform} Let $\boldw$ be a standard Brownian motion on $\mathbb{R}^{d_X}$ with $d_X = N\cdot M$, where $N \in \N^+$ is the number of time series steps and $M \in \N^+$ is the number of features tracked over time. Then $\boldv = U\boldw$ is a continuous stochastic process endowed with:\\
    \textbf{(1) Mirror Symmetry.} For all $\kappa \in [N]$, ${\boldv}_{\kappa} = {\boldv}_{N-\kappa}^*$.\\
    \textbf{(2) Real Brownian Motion.} ${\boldv}_0$ is a (real) standard Brownian motion on $\R^M$. \\
    \textbf{(3) Complex Brownian Motions.} For all $\kappa$ with $1 \leq \kappa \leq \lfloor N/2 \rfloor$, we can write ${\boldv}_{\kappa} = (\tilde{\boldw}^1_\kappa + i \tilde{\boldw}^2_\kappa) / \sqrt{2}$  where $\tilde{\boldw}^1_\kappa$ and $\tilde{\boldw}^2_\kappa$ are independent standard Brownian motions on $\R^{M}$, except when $N$ is even and $\kappa = N/2$, where   $\boldv_{N/2}$ is a real standard Brownian motion on $\mathbb{R}^{M}$. \\
    \textbf{(4) Independence.} The stochastic processes $\{\tilde{\boldv}_\kappa\}_{\kappa=0}^{\lfloor N/2 \rfloor}$ are mutually independent. \\
    We call any stochastic process satisfying the above constraints a \emph{mirrored Brownian motion} on $\C^{d_X}$. 
\end{restatable}
\begin{proof}
    The proof is given in \cref{subapp:theory_sdes}. 
\end{proof}
\begin{remark}
    Note that $\boldv$ is not strictly speaking a Brownian motion, since it contains duplicate components due to the mirror symmetry. However, our theoretical analysis in \cref{app:math_details} demonstrates that we can treat it as such by restricting to a subset of non-redundant components. 
\end{remark}

We now leverage \cref{lemma:brownian_transform} to show that $\tilde{\x}$ can be described by diffusion SDEs in the frequency domain which involve mirrored Brownian motions.

\begin{restatable}{proposition}{dftsde}(Diffusion process in frequency domain). Let us assume that $\x$ is a diffusion process that is a solution of \cref{eq:forward_process}, with $\boldG(t) = g(t) \ I_{N}$. Then \label{prop:forward_diffusion_fourier} $\tilde{\x} = \F[\x]$ is a solution to the forward diffusion process defined by:
\begin{equation} \label{eq:fourier_forward}
    \d \tilde{\x} = \tilde{\boldf}(\tilde{\x},t)\d t + g(t)\d\tilde{\boldv},
\end{equation}
where $\tilde{\boldf}(\tilde{\x},t) = U\boldf (U^*\tilde{\x},t)$ and $\tilde{\boldv}$ is a mirrored Brownian motion on $\mathbb{C}^{d_X}$.
The associated reverse diffusion process is defined by:
    \begin{equation} \label{eq:fourier_backward}
    \d \tilde{\x} = \tilde{\boldb}(\tilde{\x}, t)\d t + g(t)\d\breve{\boldv}
\end{equation}
where $ \tilde{\boldb}(\tilde{\x}, t) = \tilde{\boldf}(\tilde{\x}, t) - g^2(t) \Lambda^2 \tilde{\s}(\tilde{\x}, t)$, $\Lambda \in \R^{N \times N}$ is a diagonal matrix defined in \cref{subapp:theory_sdes}, $\d t$ is a negative infinitesimal time step,  and $\breve{\boldv}$ is a mirrored Brownian motion on $\mathbb{C}^{d_X}$ with time going from $T$ to $0$.
\end{restatable}
\begin{proof}
    The proof is given in \cref{subapp:theory_sdes}.
\end{proof}

\cref{prop:forward_diffusion_fourier} gives us a recipe to implement diffusion in the frequency domain. It guarantees that the formalism introduced by \cite{song2020score} extends to this setting with one important difference: the Brownian motion must be replaced by a mirrored Brownian motion. Ignoring this prescription by taking $\tilde{\boldv}$ to be a Brownian motion on $\C^{d_X}$ could lead to unintended consequences, such as generating complex-valued time series $\x \in \C^{d_X} \setminus \R^{d_X}$.

\subsection{Denoising score matching}
The reverse diffusion process given in \cref{eq:fourier_backward} provides an explicit way to samples time-series in the frequency domain provided we can compute $\tilde{\boldb}(\tilde{\x},t)$, which involves the unknown score $\tilde{\s}$. Like in the time domain, and motivated by \cref{eq:fourier_backward}, we build an approximation of the score with a function $\tilde{\s}_{\tilde{\theta}^*}$, whose parameters $\tilde{\theta}^*$ minimize the score matching objective: 

\begin{equation} \label{eq:score_matching_fourier}
    \tilde{\theta}^* = \arg \min_{\tilde{\theta} \in \Theta}  \mathbb{E}_{t, \tilde{\x}(0), \tilde{\x}(t)} \left[ \scorematching\left(\tilde{\s}_{\tilde{\theta}}, \boldA^{2}\tilde{\s}_{t \vert 0}, \tilde{\x}, t\right) \right]
\end{equation}

with $t \sim \mathcal{U}(0,T)$, $\tilde{\x}(0) \sim \tilde{p}_0(\tilde{\x})$, $\tilde{\x}(t) \sim \tilde{p}_{t \mid 0}(\tilde{\x}(t) \vert \tilde{\x}(0))$ and $\boldA$ is the diagonal matrix defined in \cref{prop:forward_diffusion_fourier}. In practice, this objective is evaluated by first obtaining frequency representations of time-series, and then sampling from $\tilde{p}_{t \mid 0}$ using \cref{eq:fourier_forward}. Having trained $\tilde{\s}_{\tilde{\theta}^*}$, the backward process and $\tilde{\s}_{\tilde{\theta}^*}$ permit to draw samples from $\tilde{p}_0$. It then suffices to apply the inverse DFT $\F^{-1}$ to map the resulting complex-valued signals back into the time domain. 

One important question remains at this stage. How does training a score in the frequency domain allow to generate DFT of time series sampled from $p_\mathrm{data}$? In other words, how does minimizing the score matching in \cref{eq:score_matching_fourier} imply that $\tilde{p}_0 \approx \tilde{p}_{\mathrm{data}}$? To answer this question, a key observation is that we can associate an auxiliary score $\s'_{\tilde{\theta}}$ in the time domain to the score $\tilde{\s}_{\tilde{\theta}}$ by applying an inverse DFT $\F^{-1}$. Below, we show that minimizing the score matching loss from \cref{eq:score_matching_fourier} for the score $\tilde{\s}_{\tilde{\theta}}$ is equivalent to minimizing the score matching loss from \cref{eq:score_matching} for the auxiliary score $\s'_{\tilde{\theta}}$. This important observation connects the reverse diffusion process in the frequency domain described by \cref{eq:fourier_backward} with a reverse diffusion process in the time domain following \cref{eq:backward_process}. 

\begin{restatable}{proposition}{dftscore}(Score matching equivalence).\label{prop:parseval}
Consider a score $\tilde{\s}_{\tilde{\theta}} : \C^{d_X} \times [0, T] \rightarrow \C^{d_X}$ defined in the frequency domain and satisfying the mirror symmetry $[\tilde{\s}_{\tilde{\theta}}]_{\kappa} = [\tilde{\s}_{\tilde{\theta}}^*]_{N- \kappa}$ for all $\kappa \in [N]$. Let us define an auxiliary score $\s'_{\tilde{\theta}}: \R^{d_X} \times [0, T] \rightarrow \R^{d_X}$ as $(\x, t) \mapsto \s'_{\tilde{\theta}}(\x, t) = U^* \tilde{\s}_{\tilde{\theta}}(U \x, t) $  in the time domain.
The score matching loss in the frequency domain is equivalent to the score matching loss for the auxiliary score in the time domain:
\begin{align} \label{eq:parseval}
  \scorematching\left(\tilde{\s}_{\tilde{\theta}}, \Lambda^{2} \tilde{\s}_{t \vert 0}, \tilde{\x}, t\right) = \scorematching\left(\s'_{\tilde{\theta}}, \s_{t \vert 0}, \x, t\right)
\end{align}
where $\tilde{\s}_{t \vert 0}(\tilde{\x},t) = \nabla_{\tilde{\x}(t)}\log \tilde{p}_{t \vert 0}(\tilde{\x}(t) \vert \tilde{\x}(0))$ ,  $\s_{t \vert 0}(\x, t) = \nabla_{\x(t)}\log p_{t \vert 0}(\x(t) \vert \x(0))$, and $\boldA$ is the diagonal matrix in \cref{prop:forward_diffusion_fourier}.
\end{restatable}
\begin{proof}
    The proof is given in \cref{subapp:theory_score}. 
\end{proof}

\cref{prop:forward_diffusion_fourier,prop:parseval} provide an explicit way to translate diffusion in the time domain to diffusion in the frequency domain. We note that attempting to solve \cref{eq:score_matching,eq:score_matching_fourier} in the finite-sample regime yields a local minimum solution in practice. Hence, there is no guarantee that training a score model in the frequency domain will converge to an auxiliary score $\s'_{\tilde{\theta}^*} = \s_{\theta^*}$ identical to the one obtained by training the score model in the time domain. In particular, having a score function $\tilde{\s}_{\theta}$ defined in the frequency domain is an important inductive bias, which is likely to alter the training dynamic. Through our experiments in the next section, we study the effect of this inductive bias on the resulting diffusion processes.

\begin{mdframed}[style=takeaway]
\takeaway{1} Diffusion in the frequency domain can be implemented by replacing the standard Brownian motions with mirrored Brownian motions in the diffusion SDEs. The associated score can be optimized by minimizing a denoising score matching loss.
\end{mdframed}

\section{Comparing time and frequency diffusion} \label{sec:empirical}
\begin{table*}[ht] 
\centering 
\caption{Various datasets used in our experiments and some of their properties.} \label{tab:dataset_description}
\begin{tabular}{|lllrrr|}
\hline
\textit{Dataset} & \textit{Reference}  & \textit{Field} &\textit{\# Samples} & \textit{\# Steps $N$} & \textit{\# Features $M$} \\
\hline
ECG & \cite{kachuee2018ecg} & \multirow[c]{2}{*}{Healthcare} & 87,553 & 187 & 1 \\
MIMIC-III & \cite{johnson2016mimic} & & 19,155 & 24 & 40  \\  
\rowcolor{gray!15}
NASDAQ-2019 & \cite{onyshchak2020} & Finance   & 4,827 & 252 & 5 \\ 
NASA-Charge & \multirow[c]{2}{*}{\cite{saha2007}} & \multirow[c]{2}{*}{Engineering} & 2,396 & 251 & 4 \\
NASA-Discharge & & & 1,755 & 134 & 5 \\ 
\rowcolor{gray!15}
US-Droughts & \cite{minixhofer2021} & Climate & 2,797 & 365 & 13 \\ 
\hline
\end{tabular}
\end{table*}

\begin{table*}[ht]
\centering
\caption{Sliced Wasserstein distances ($\downarrow$) evaluated in the time domain ($SW(\D{train}, \S{time})$, $SW(\D{train}, \S{freq})$) and in the frequency domain ($SW(\tildeD{train}, \tildeS{time})$, $SW(\tildeD{train}, \tildeS{freq})$) on the various datasets. For each distance, we report its mean $\pm$ 2 standard errors.}
\label{tab:sliced_wasserstein}
\begin{tabular}{|l|e|gg|}
\hline
\rowcolor{white}
 \textit{Dataset} & \textit{Metric Domain} & \multicolumn{2}{c|}{\textit{Diffusion Domain}} \\
 \rowcolor{white}
 &  & Frequency & Time \\ \hline
 \rowcolor{white}
\multirow[t]{2}{*}{ECG} & Frequency & $\mathbf{0.012 \ \pm \ 0.000}$ & $0.020 \ \pm \ 0.000$ \\
 & Time & $\mathbf{0.015 \ \pm \ 0.000}$ & $0.021 \ \pm \ 0.000$ \\
 \rowcolor{white}
\multirow[t]{2}{*}{MIMIC-III} & Frequency & $\mathbf{0.144 \ \pm \ 0.004}$ & $0.206 \ \pm \ 0.006$ \\
 & Time & $\mathbf{0.152 \ \pm \ 0.004}$ & $0.211 \ \pm \ 0.006$ \\
 \rowcolor{white}
 \multirow[t]{2}{*}{NASDAQ-2019} & Frequency & $\mathbf{45.812 \ \pm \ 2.096}$ & $64.056 \ \pm \ 3.040$ \\
 & Time & $\mathbf{43.602 \ \pm \ 2.044}$ & $60.512 \ \pm \ 2.960$ \\
 \rowcolor{white}
\multirow[t]{2}{*}{NASA-Charge} & Frequency & $\mathbf{0.211 \ \pm \ 0.008}$ & $0.27 \ \pm \ 0.006$ \\
 & Time & $\mathbf{0.229 \ \pm \ 0.008}$ & $0.316 \ \pm \ 0.008$ \\
 \rowcolor{white}
\multirow[t]{2}{*}{NASA-Discharge} & Frequency & $\mathbf{1.999 \ \pm \ 0.084}$ & $2.974 \ \pm \ 0.134$ \\
 & Time & $\mathbf{2.028 \ \pm \ 0.082}$ & $2.942 \ \pm \ 0.134$ \\
 \rowcolor{white}
\multirow[t]{2}{*}{US-Droughts} & Frequency & $\mathbf{0.633 \ \pm \ 0.018}$ & $2.849 \ \pm \ 0.090$ \\
 & Time & $\mathbf{0.738 \ \pm \ 0.020}$ & $2.913 \ \pm \ 0.092$ \\
\hline
\end{tabular}
\end{table*}

\begin{figure*}[ht]
\centering
    \begin{subfigure}{0.4\textwidth}
        \includegraphics[width=\textwidth]{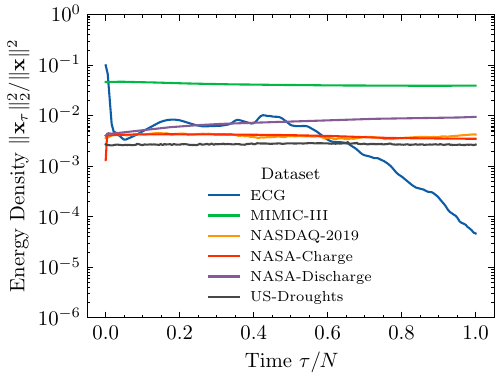}
        \caption{Time Domain.}
        \label{subfig:energy_density_datasets}
    \end{subfigure}
    \hfill
    \begin{subfigure}{0.4\textwidth}
        \includegraphics[width=\textwidth]{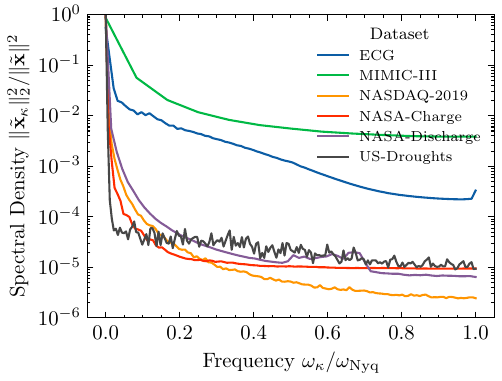}
        \caption{Frequency Domain.}
        \label{subfig:spectral_density_datasets}
    \end{subfigure}
    \caption{Localization of time series in the time and frequency domains. Time series are more localized in the frequency domain.}
    \label{fig:energy_spectrum_datasets}
\end{figure*}

\begin{figure}[ht]
    \centering
\includegraphics[width=.6\linewidth]{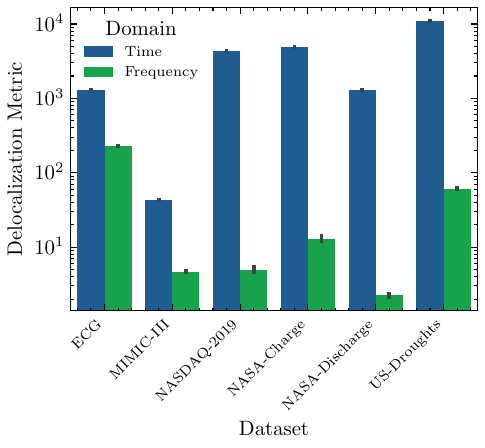}
\caption{By evaluating our delocalization metrics in the time domain ({$\color{NavyBlue} \deloc{time}$}) and the frequency domain ($\color{ForestGreen} \deloc{freq}$), we quantitatively confirm that all the datasets are significantly more localized in the frequency domain. All the metrics are averaged over $\D{train}$, their mean is reported with a $95\%$ confidence interval.}
\label{fig:delocalization_datasets}
\end{figure}

\begin{figure*}[ht]
\centering
    \begin{subfigure}{0.45\textwidth}
        \includegraphics[width=\textwidth]{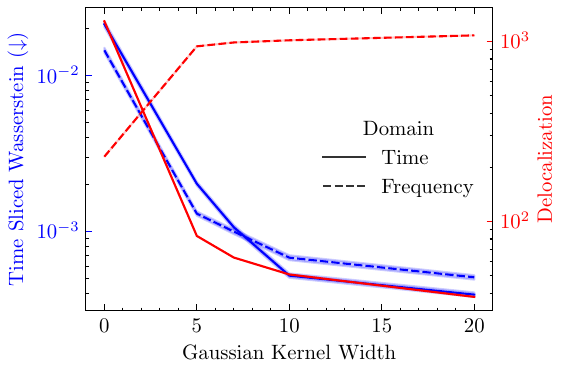}
        \caption{Time Domain.}
        \label{subfig:ecg_conv_time}
    \end{subfigure}
    \hfill
    \begin{subfigure}{0.45\textwidth}
        \includegraphics[width=\textwidth]{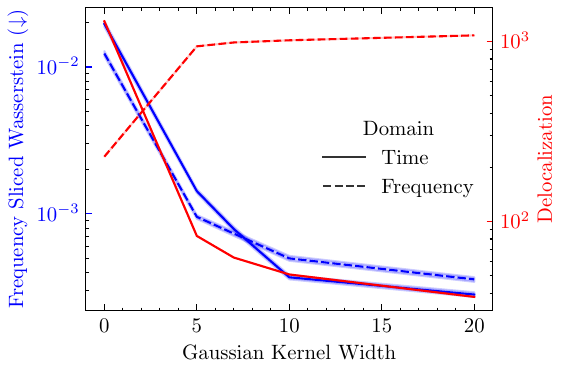}
        \caption{Frequency Domain.}
        \label{subfig:ecg_conv_freq}
    \end{subfigure}
    \caption{Sliced Wasserstein distances of time and frequency models (\textcolor{blue}{blue}) and localization metrics in time and frequency domains (\textcolor{red}{red}) when smoothing the spectral representations of the time series with Gaussian kernels of variable width. Increasing the kernel width removes the localization in the frequency domain and increases the localization in the time domain. Coincidentally, the time diffusion model becomes better than the frequency diffusion model.}
    \label{fig:ecg_conv} 
\end{figure*}

In this section, we empirically analyze the effect of performing time series diffusion in the frequency domain.  In \cref{subsec:sample_quality}, we show that frequency diffusion models better capture their training distribution than time models. In \cref{subsec:sample_differences}, we argue that these differences of performance can be attributed the localization of the time series in the frequency domain. Finally, in \cref{subsec:when_frequency}, we artificially create settings where time models outperform the frequency models in order to confirm this hypothesis. The code necessary to reproduce the results, along with detailed instructions, is provided in the following repository: \url{https://github.com/JonathanCrabbe/FourierDiffusion}.

\textbf{Data.} To illustrate the breadth of time series applications, we work with 6 different datasets described in \cref{tab:dataset_description}. We observe that these datasets cover many use-cases (healthcare, finance, engineering and climate modelling), sample sizes, sequence lengths $N$ and  number of features tracked over time $M$. All the datasets are standardized before being fed to models. We also split the datasets into a training set $\D{train}$ and a validation set $\D{val}$. We provide more details on the datasets in \cref{subapp:datasets_details}.

\textbf{Models.} For each dataset, we parametrize the time score model $\s_{\theta}$ and the frequency score model $\tilde{\s}_{\tilde{\theta}}$ as transformer encoders with 10 attention and MLP layers, each with 12 heads and dimension $d_{\mathrm{model}} = 72$. Both models have learnable positional encoding as well as diffusion time $t$ encoding through random Fourier features composed with a learnable dense layer. This results in models with $3.2$M parameters. We use a VP-SDE with linear noise scheduling and $\beta_{\mathrm{min}} = 0.1$ and $\beta_{\mathrm{max}} = 20$, as in \cite{song2020score}. The score models are trained with the denoising score-matching loss, as defined in \cref{sec:theory}. All the models are trained for 200 epochs with batch size $64$, AdamW optimizer and cosine learning rate scheduling ($20$ warmup epochs, $\mathrm{lr_{max}} = 10^{-3}$). The selected model is the one achieving the lowest validation loss.

\textbf{Time and frequency.} Crucially, the only difference between the time and the fequency diffusion models is the domain in which their input time series are represented. Since all datasets are expressed in the time domain, they can directly be fed to the time diffusion model $\s_{\theta}$. When it comes to the frequency diffusion model $\tilde{\s}_{\tilde{\theta}}$, the data is first mapped to the frequency domain by applying a DFT $\F$ on each time series. In the time domain, the forward and reverse diffusion obey the SDEs in \cref{eq:forward_process,eq:backward_process}. In the frequency domain, the forward and reverse diffusion obey the modified SDEs in \cref{eq:fourier_forward,eq:fourier_backward}. The denoised samples $\tilde{\x}(0)$ obtained in the frequency domain can be pulled back to the time domain by applying an inverse DFT $\tilde{\x}(0) \mapsto \F^{-1}[\tilde{\x}(0)]$. In the following, we shall denote by $\S{time} \subset \R^{d_X}$ and $\S{freq} \subset \R^{d_X}$ the time representation of the samples generated by the time and frequency models. Similarly, we shall denote by $\tildeS{time} := \F[\S{time}]$ and $\tildeS{freq} := \F[\S{freq}]$ the frequency representations of these time series. We sample $|\S{time}| = |\S{freq}| = 10,000$ samples for each model by applying $T = 1,000$ diffusion time steps.

\subsection{Which samples better capture the distribution?} \label{subsec:sample_quality}

\textbf{Methodology.} We are interested in the faithfulness of the samples generated by the time and frequency diffusion models. Ideally, this faithfulness should be evaluated by computing the Wasserstein distance between the true distribution and the distribution spanned by our diffusion models. However, this is impossible since the exact computation of the Wasserstein distance in intractable in input spaces of large dimension $d_X \gg 1$. In the case of images, \cite{heusel2017gans} mitigates these problems by mapping all the images in a lower dimensional representation space (the activation of the penultimate layer of an Inception-V3 model). This crucially relies on the fact that the Inception-V3 provides high-quality representations of images. Unfortunately, such a general representation of time series does not exist in practice. Hence, our evaluation needs to be performed in the input space $\R^{d_X}$ directly. For this reason, we shall rely on the sliced Wasserstein distance introduced by \cite{bonneel2015sliced}, which has similar properties to the Wasserstein distance and can be efficiently estimated in high dimension spaces. With a slight abuse of notation, we shall denote by $SW(\S{1}, \S{2})$ the sliced Wasserstein distances between the empirical distributions corresponding to the samples $\S{1}$ and $\S{2}$. Its detailed definition is provided in \cref{subapp:eval_details}. 

\textbf{Analysis.} The sliced Wasserstein distances are reported in \cref{tab:sliced_wasserstein}. Interestingly, we observe that the frequency diffusion models consistently outperform the time diffusion models for all datasets both in the time and the frequency domain. In \cref{subapp:more_plots}, we show that using marginal Wasserstein distances instead of sliced Wasserstein distances essentially leads to the same conclusion. In order to verify that our observations are not specific to transformer backbones, we reproduce the same experiment with LSTM backbones in \cref{app:other_backbone} and obtain similar results showing the superior performance of frequency diffusion models. While this observation already suggests the benefits of diffusing in the frequency domain rather than in the time domain, it is important to understand how these performance gains emerge. For this, we need to gain a better understanding of the training distributions, which is the object of next section.

\subsection{How to explain the differences?} \label{subsec:sample_differences}

\textbf{Signal energy analysis.} Before formulating an hypothesis as to why the frequency models are better, it is helpful to gain a better understanding about the training distribution $\D{train}$. To that end, we leverage the energy and spectral densities related to the time series, described in \cref{sec:background}. These densities are represented in \cref{fig:energy_spectrum_datasets}, where we have averaged the densities over all time series in $\D{train}$. By analyzing \cref{subfig:spectral_density_datasets}, we make a key observation: for all datasets, most of the time series energy in the frequency domain is localized on the frequency $\omega_0 = 0$ also known as the \emph{fundamental frequency}. Furthermore, we observe that the energy quickly decays as the frequency increases. This observation suggests that the low frequencies capture most of the time series information. This is in stark contrast with the energy distribution over time in \cref{subfig:energy_density_datasets}, which is more uniform over time for all the datasets. This asymmetry between the time series spectral localization and their temporal delocalization is a promising candidate to explain the superior performances of frequency diffusion. We now make this observation more quantitative.

\textbf{Quantitative signal localization.} In order to measure how delocalized a time series $\x \in \R^{d_X}$ is in the time domain, we shall use the \emph{delocalization metrics} introduced by \cite{nam2013}:
\begin{align} \label{eq:delocalization}
    \deloc{time}(\x) := \min_{\tau \in [N]} \frac{1}{\| \x \|^2} \sum_{\tau' \in [N]} d_{\mathrm{cyc}}(\tau , \tau') \| \x_{\tau'} \|_2^2, 
\end{align}
where $d_{\mathrm{cyc}}: [N]^2 \rightarrow [N]$ is the cyclic distance defined as $d_{\mathrm{cyc}}(\tau , \tau') = \min (|\tau - \tau'|, N - |\tau - \tau'|)$, $\| \cdot\|$ denotes the Frobenius norm and $\| \cdot \|_2$ the Euclidean norm. Similarly, we can compute the delocalization in the frequency domain $\deloc{freq}$ by replacing $\x \mapsto \tilde{\x}$ and $\tau , \tau' \mapsto \kappa , \kappa'$ in \cref{eq:delocalization}. We report these delocalization metrics for each dataset in \cref{fig:delocalization_datasets}. This quantitative analysis confirms our previous observations: the time series in all the datasets appear significantly more localized in the frequency domain. Interestingly, we never observe a time series that is localized in both the frequency and time domain simultaneously. This is in agreement with the \emph{uncertainty principle} from \cite{nam2013}, which echoes the foundational work of \cite{heisenberg1927}. This verification is made in \cref{subapp:more_plots}.

\textbf{A localization explanation.} Based on the previous observation, we postulate that higher localization of the time series in the frequency domain contributes to the superior performance of frequency diffusion models. While we will test this hypothesis in the next section, it is useful to discuss the intuition behind it first. Due to the frequency localization, the frequency score model is presented with a representation of the time series where most of the relevant information is aligned with few components of the model's input (i.e. the lower frequencies, especially the fundamental). This is in contrast with the time model, which is presented with an input where all the components matter equally. It follows that the frequency model does not need to learn a good distribution over all frequencies in order to generate samples of high quality, provided the lower frequency distributions are properly learned. The time model, on the other hand, needs to model all the time steps accurately in order to generate high-quality samples. If this intuition is correct, it would imply that delocalizing the signal in the frequency domain should reduce the gap of performance between time and frequency models. This is analyzed in the next section.

\subsection{Should we always diffuse in the frequency domain?} \label{subsec:when_frequency}

\textbf{Removing spectral localization.} We would like to assess whether the localization of these signals in the frequency domain contribute to explain the superior performances of frequency diffusion models over time diffusion models. To test this hypothesis, we intervene on a given dataset with the objective of varying the localization of frequency and time representations. With this in mind, it is useful to start from a dataset where the imbalance between time and frequency localization is not severe. By looking at \cref{fig:delocalization_datasets}, it is clear that the ECG dataset is the best candidate. In order to gradually remove the spectral localization from the ECG dataset $\D{ECG}$, we convolve the time series $\x$ in the frequency domain with Gaussians of increasing kernel width $\sigma \in \R^+$ and define $\x^{\sigma} := \F^{-1}  \left[ \F[\x] \ \star \ \mathbf{g}^\sigma \right] $, where $\star$ denotes the convolution between two signals and $g^{\sigma}_\kappa := Z^{-1} \exp[- \kappa^2 /(2 \sigma^2)]$ for all $\kappa \in [N]$ is a Gaussian kernel with normalization $Z = \sum_{\kappa=1}^{N} \exp[- \kappa^2 /(2 \sigma^2)]$. This results in a family of corrupted datasets $\D{ECG}^\sigma$, where the localization in the frequency domain decreases as $\sigma$ increases. This is indeed what we observe in \cref{fig:ecg_conv} with the \textcolor{red}{red curves}. Coincidentally, the delocalization decreases in the time domain, in agreement with the uncertainty principle. The two curves cross at $\sigma \approx 2$, beyond which the time series are more localized in the time domain. Let us now analyze how the model performances evolve with different values of $\sigma$. 

\textbf{Analysis.} We train time and frequency diffusion models on the datasets $\D{ECG}^\sigma$ for $\sigma \in \{ 0, 5, 7, 10, 20\}$, where $\sigma=0$ corresponds to the original ECG dataset: $\D{ECG}^{\sigma = 0} = \D{ECG}$. As in \cref{subsec:sample_quality}, we measure the quality of $10,000$ samples $\S{}$ produced by these models with the Wasserstein distances $SW(\D{ECG}^\sigma, \S{})$ in \cref{subfig:ecg_conv_time} and $SW(\tildeD{ECG}^\sigma, \tildeS{})$ in \cref{subfig:ecg_conv_freq}. By inspecting the \textcolor{blue}{blue curves} from these figures, we notice that decreasing the frequency closes the gap between the time and the frequency models. Moreover, the two curves cross around at $\sigma \approx 7$, beyond which the time model outperforms the frequency model. This confirms our hypothesis that the localization (at least) partially explains the better performance of the frequency diffusion model.

\textbf{A cautionary remark.} Before concluding, we would like to incorporate a bit of nuance in our analysis. While the above results suggest that localization is an important factor to explain the superior performance of frequency diffusion models, we \emph{do not} claim that this is the only explanation. There are essentially two things that suggest that this only partially explains the observations from \cref{subsec:sample_quality}. (1)~In \cref{fig:ecg_conv}, the blue curves and the red curves don't cross for the same value of $\sigma$ (i.e. the red curves cross at $\sigma \approx 2$ and the blue curves at $\sigma \approx 7$). Hence, the time model requires time series that are substantially more localized in the time domain in order to outperform the frequency model. (2)~In the limit $\sigma \rightarrow + \infty$, the time series become constant in the frequency domain. While this corresponds to a minimal localization in the frequency model, this should also be very easy for a frequency diffusion model to learn. Hence, decreasing the localization of the time series in a given domain does not necessarily imply that the resulting time series are more difficult to model in that domain.

\begin{mdframed}[style=takeaway]
\takeaway{2} For all datasets in this paper, diffusing in the frequency domain yields better performances than in the time domain. A promising explanation for this is that time series from these datasets are substantially more localized in their spectral representation than in their temporal representation.
\end{mdframed}

\section{Discussion}
\label{sec:discussion}
In this work, we have improved the understanding of how diffusion models should be used with time series. We constructed a theoretical framework that extends the score-based SDE formulation of diffusion to complex-valued times series representations in the frequency domain. We have then demonstrated empirically that implementing time series diffusion in the frequency domain consistently outperforms the canonical diffusion in the time domain. Finally, we showed that the spectral localization of the time series plays a significant role to explain this phenomenon. There is a number of interesting ways to extend our work.

\textbf{Time localized datasets.} While all 6 datasets we have studied appear substantially more localized in their spectral representation, we do not claim that is is a universal property of real-world time series. In particular, it would be of interest to survey a large amount of time series datasets to determine the extent to which this phenomenon occurs.

\textbf{Latent diffusion.} Latent diffusion has emerged as a fruitful direction of research in the diffusion literature \cite{rombach2022video}. A promising direction would be to study how spectral representations of time series can be incorporated to their latent representations and whether this benefits the quality of the generated samples. We leave these insightful research directions for future work.

\clearpage
\section*{Acknowledgments}
Jonathan Crabbé and Jan Stanczuk are funded by Aviva, Nicolas Huynh by Illumina  and Mihaela van der Schaar by the Office of Naval Research (ONR), NSF 172251.

\bibliographystyle{unsrt}  
\bibliography{bibliography}  

\begin{thebibliography}{10}

\bibitem{karras2020stylegan2}
Tero Karras, Samuli Laine, Miika Aittala, Janne Hellsten, Jaakko Lehtinen, and
  Timo Aila.
\newblock Analyzing and improving the image quality of stylegan.
\newblock In {\em Proceedings of the IEEE/CVF conference on computer vision and
  pattern recognition}, pages 8110--8119, 2020.

\bibitem{dhariwal2021diffusion_beat_GANs}
Prafulla Dhariwal and Alexander Nichol.
\newblock Diffusion models beat gans on image synthesis.
\newblock {\em Advances in neural information processing systems},
  34:8780--8794, 2021.

\bibitem{kong2021diffwave}
Zhifeng Kong, Wei Ping, Jiaji Huang, Kexin Zhao, and Bryan Catanzaro.
\newblock Diffwave: {A} versatile diffusion model for audio synthesis.
\newblock In {\em 9th International Conference on Learning Representations,
  {ICLR} 2021, Virtual Event, Austria, May 3-7, 2021}, 2021.

\bibitem{donahue2018adversarial_audio}
Chris Donahue, Julian McAuley, and Miller Puckette.
\newblock Adversarial audio synthesis.
\newblock In {\em International Conference on Learning Representations}, 2018.

\bibitem{rombach2022video}
Robin Rombach, Andreas Blattmann, Dominik Lorenz, Patrick Esser, and Bj{\"o}rn
  Ommer.
\newblock High-resolution image synthesis with latent diffusion models.
\newblock In {\em Proceedings of the IEEE/CVF conference on computer vision and
  pattern recognition}, pages 10684--10695, 2022.

\bibitem{dieleman2022continuous}
Sander Dieleman, Laurent Sartran, Arman Roshannai, Nikolay Savinov, Yaroslav
  Ganin, Pierre~H Richemond, Arnaud Doucet, Robin Strudel, Chris Dyer, Conor
  Durkan, et~al.
\newblock Continuous diffusion for categorical data.
\newblock {\em arXiv preprint arXiv:2211.15089}, 2022.

\bibitem{lugmayr2022repaint}
Andreas Lugmayr, Martin Danelljan, Andres Romero, Fisher Yu, Radu Timofte, and
  Luc Van~Gool.
\newblock Repaint: Inpainting using denoising diffusion probabilistic models.
\newblock In {\em Proceedings of the IEEE/CVF Conference on Computer Vision and
  Pattern Recognition}, pages 11461--11471, 2022.

\bibitem{saharia2022sr3}
Chitwan Saharia, Jonathan Ho, William Chan, Tim Salimans, David~J Fleet, and
  Mohammad Norouzi.
\newblock Image super-resolution via iterative refinement.
\newblock {\em IEEE Transactions on Pattern Analysis and Machine Intelligence},
  45(4):4713--4726, 2022.

\bibitem{watson2023rfdiffusion}
Joseph~L Watson, David Juergens, Nathaniel~R Bennett, Brian~L Trippe, Jason
  Yim, Helen~E Eisenach, Woody Ahern, Andrew~J Borst, Robert~J Ragotte, Lukas~F
  Milles, et~al.
\newblock De novo design of protein structure and function with rfdiffusion.
\newblock {\em Nature}, 620(7976):1089--1100, 2023.

\bibitem{xu2022geodiff}
Minkai Xu, Lantao Yu, Yang Song, Chence Shi, Stefano Ermon, and Jian Tang.
\newblock Geodiff: a geometric diffusion model for molecular conformation
  generation.
\newblock {\em ArXiv}, abs/2203.02923, 2022.

\bibitem{zeni2023mattergen}
Claudio Zeni, Robert Pinsler, Daniel Z{\"u}gner, Andrew Fowler, Matthew Horton,
  Xiang Fu, Sasha Shysheya, Jonathan Crabb{\'e}, Lixin Sun, Jake Smith, et~al.
\newblock Mattergen: a generative model for inorganic materials design.
\newblock {\em arXiv preprint arXiv:2312.03687}, 2023.

\bibitem{gatta2022nn_time_series}
Federico Gatta, Fabio Giampaolo, Edoardo Prezioso, Gang Mei, Salvatore Cuomo,
  and Francesco Piccialli.
\newblock Neural networks generative models for time series.
\newblock {\em J. King Saud Univ. Comput. Inf. Sci.}, 34:7920--7939, 2022.

\bibitem{yoon2019timegan}
Jinsung Yoon, Daniel Jarrett, and Mihaela Van~der Schaar.
\newblock Time-series generative adversarial networks.
\newblock {\em Advances in neural information processing systems}, 32, 2019.

\bibitem{alaa2021fourierflows}
Ahmed~M. Alaa, Alex~J. Chan, and Mihaela van~der Schaar.
\newblock Generative time-series modeling with fourier flows.
\newblock In {\em International Conference on Learning Representations}, 2021.

\bibitem{esteban2017rcgan}
Cristóbal Esteban, Stephanie~L. Hyland, and Gunnar Rätsch.
\newblock Real-valued (medical) time series generation with recurrent
  conditional gans, 2017.

\bibitem{hyvarinen2005estimation}
Aapo Hyv{\"a}rinen and Peter Dayan.
\newblock Estimation of non-normalized statistical models by score matching.
\newblock {\em Journal of Machine Learning Research}, 6(4), 2005.

\bibitem{sohl2015diffusion}
Jascha Sohl-Dickstein, Eric Weiss, Niru Maheswaranathan, and Surya Ganguli.
\newblock Deep unsupervised learning using nonequilibrium thermodynamics.
\newblock In {\em International conference on machine learning}, pages
  2256--2265. PMLR, 2015.

\bibitem{ho2020ddpm}
Jonathan Ho, Ajay Jain, and Pieter Abbeel.
\newblock Denoising diffusion probabilistic models.
\newblock {\em Advances in neural information processing systems},
  33:6840--6851, 2020.

\bibitem{song2020score}
Yang Song, Jascha Sohl-Dickstein, Diederik~P Kingma, Abhishek Kumar, Stefano
  Ermon, and Ben Poole.
\newblock Score-based generative modeling through stochastic differential
  equations.
\newblock In {\em International Conference on Learning Representations}, 2020.

\bibitem{lin2023times_series_diffusion_survey}
Lequan Lin, Zhengkun Li, Ruikun Li, Xuliang Li, and Junbin Gao.
\newblock Diffusion models for time-series applications: a survey.
\newblock {\em Frontiers of Information Technology \& Electronic Engineering},
  pages 1--23, 2023.

\bibitem{korner2022fourier}
Thomas~William K{\"o}rner.
\newblock {\em Fourier analysis}.
\newblock Cambridge university press, 2022.

\bibitem{yi2023deep_learning_ts_fourier_survey}
Kun Yi, Qi~Zhang, Longbing Cao, Shoujin Wang, Guodong Long, Liang Hu, Hui He,
  Zhendong Niu, Wei Fan, and Hui Xiong.
\newblock A survey on deep learning based time series analysis with frequency
  transformation.
\newblock {\em CoRR, abs/2302.02173}, 2023.

\bibitem{shin2023frequency}
Donghyeok Shin, Seungjae Shin, and Il-Chul Moon.
\newblock Frequency domain-based dataset distillation.
\newblock {\em arXiv preprint arXiv:2311.08819}, 2023.

\bibitem{phillips2022spectral}
Angus Phillips, Thomas Seror, Michael Hutchinson, Valentin De~Bortoli, Arnaud
  Doucet, and Emile Mathieu.
\newblock Spectral diffusion processes.
\newblock {\em arXiv preprint arXiv:2209.14125}, 2022.

\bibitem{anderson1982}
Brian~D.O. Anderson.
\newblock Reverse-time diffusion equation models.
\newblock {\em Stochastic Processes and their Applications}, 12(3):313--326,
  1982.

\bibitem{song2019generative}
Yang Song and Stefano Ermon.
\newblock Generative modeling by estimating gradients of the data distribution.
\newblock {\em Advances in neural information processing systems}, 32, 2019.

\bibitem{vincent2011connection}
Pascal Vincent.
\newblock A connection between score matching and denoising autoencoders.
\newblock {\em Neural computation}, 23(7):1661--1674, 2011.

\bibitem{schreier2010}
Peter~J. Schreier and Louis~L. Scharf.
\newblock {\em Statistical Signal Processing of Complex-Valued Data: The Theory
  of Improper and Noncircular Signals}.
\newblock Cambridge University Press, 2010.

\bibitem{kachuee2018ecg}
Mohammad Kachuee, Shayan Fazeli, and Majid Sarrafzadeh.
\newblock Ecg heartbeat classification: A deep transferable representation.
\newblock In {\em 2018 IEEE international conference on healthcare informatics
  (ICHI)}, pages 443--444. IEEE, 2018.

\bibitem{johnson2016mimic}
Alistair Johnson, Tom Pollard, and Roger Mark.
\newblock Mimic-iii clinical database (version 1.4).
\newblock {\em PhysioNet}, 10(C2XW26):2, 2016.

\bibitem{onyshchak2020}
Oleh Onyshchak.
\newblock Stock market dataset, 2020.

\bibitem{saha2007}
B~Saha and T.~Goebel.
\newblock Battery data set, nasa ames prognostics data repository, 2007.

\bibitem{minixhofer2021}
Christoph Minixhofer.
\newblock Predict droughts using weather and soil data, 2021.

\bibitem{heusel2017gans}
Martin Heusel, Hubert Ramsauer, Thomas Unterthiner, Bernhard Nessler, and Sepp
  Hochreiter.
\newblock Gans trained by a two time-scale update rule converge to a local nash
  equilibrium.
\newblock {\em Advances in neural information processing systems}, 30, 2017.

\bibitem{bonneel2015sliced}
Nicolas Bonneel, Julien Rabin, Gabriel Peyr{\'e}, and Hanspeter Pfister.
\newblock Sliced and radon wasserstein barycenters of measures.
\newblock {\em Journal of Mathematical Imaging and Vision}, 51:22--45, 2015.

\bibitem{nam2013}
Sangnam Nam.
\newblock {An Uncertainty Principle for Discrete Signals}.
\newblock In {\em {SampTA}}, Bremen, Germany, 2013.

\bibitem{heisenberg1927}
Werner Heisenberg.
\newblock {\"U}ber den anschaulichen inhalt der quantentheoretischen kinematik
  und mechanik.
\newblock {\em Zeitschrift f{\"u}r Physik}, 43(3):172--198, 1927.

\bibitem{kloeden1992stochastic}
Peter~E Kloeden, Eckhard Platen, Peter~E Kloeden, and Eckhard Platen.
\newblock {\em Stochastic differential equations}.
\newblock Springer, 1992.

\bibitem{wang2020mimic}
Shirly Wang, Matthew~BA McDermott, Geeticka Chauhan, Marzyeh Ghassemi,
  Michael~C Hughes, and Tristan Naumann.
\newblock Mimic-extract: A data extraction, preprocessing, and representation
  pipeline for mimic-iii.
\newblock In {\em Proceedings of the ACM conference on health, inference, and
  learning}, pages 222--235, 2020.

\end{thebibliography}

\newpage
\appendix

\section{Mathematical details} \label{app:math_details}
\subsection{DFT properties} \label{subapp:theory_dft}

\begin{proposition}[Unitarity of the DFT operator] \label{prop:unitary_U} The DFT matrix $U \in \C^{N \times N}$ with elements $[U]_{\kappa \tau} := N^{- 1 / 2} \exp (- i \omega_{\kappa} \tau )$ with $\omega_{\kappa} := \frac{\kappa 2 \pi}{N}$ is unitary.
\end{proposition}
\begin{proof}
Let $U^*$ denote the conjugate transpose of $U$.
    For any $\kappa$ and $\tau$ in $[N]$, we have:
    
    \begin{align*}
    [UU^*]_{\kappa \tau} 
    &= \sum_{\beta = 0}^{N-1}[U]_{\kappa \beta}[U^*]_{\beta \tau} \\
    &= \frac{1}{N} \sum_{\beta = 0}^{N-1}\exp(-i\omega_{\kappa} \beta ) \exp(i \omega_{\beta} \tau) \\
    &= \frac{1}{N} \sum_{\beta=0}^{N-1}\exp(-i\omega_{\kappa - \tau}\beta)
      \end{align*}
Hence, if $\kappa = \tau$, we have $[UU^*]_{\kappa \tau} = 1$, otherwise $[UU^*]_{\kappa \tau} = 0$, since $\exp(-i\omega_{\kappa-\tau}N) = 1$. This is equivalent to $UU^* = I_{N}$, i.e. $U$ is unitary.
\end{proof}

\begin{proposition} \label{prop:symmetry}
The DFT $\tilde{\x} = \F[\x] = U \x$  of a real-valued time series $\x \in \R^{d_X}$ verifies the following \emph{mirror symmetry} for all $\kappa \in [N]$: 
\begin{equation*}
    \tilde{\x}_{\kappa} = \tilde{\x}_{N-\kappa}^*.
\end{equation*}
\end{proposition}
\begin{proof}
    Let $\kappa$ and $\tau$ be in $[N]$. We first note that $\exp(i\omega_{N-\kappa}\tau) = \exp(i (\omega_{N} -\omega_{\kappa}) \tau) = \exp(-i \omega_{\kappa} \tau)$. Hence,  
    \begin{align*}
    \tilde{\x}_{N-\kappa}^* &= \sum_{\tau = 0}^{N-1} [U]_{N-\kappa, \tau}^{*} \x_{\tau} &&\text{($\x$ is real valued)} \\ 
    &=N^{-1/2}\sum_{\tau = 0}^{N-1} \exp(i\omega_{N-\kappa}\tau) \x_{\tau} \\
    &= N^{-1/2}\sum_{\tau = 0}^{N-1} \exp(-i\omega_{\kappa}\tau) \x_{\tau} \\
    &= \sum_{\tau=0}^{N-1} [U]_{\kappa, \tau} \x_{\tau} \\
    &= \tilde{\x}_{\kappa}
    \end{align*}
\end{proof}

\subsection{Densities and scores for constrained signals} \label{subapp:theory_submanifold}

As we have mentioned in \cref{sec:background}, the redundancy between certain components of the DFT $\tilde{\x} \in \C^{d_X}$ of $\x$ expressed by \cref{eq:real_constraint} needs to be taken into account if we wish to define a density $\tilde{p}$ for the time series distribution in the frequency domain. In particular, this redundancy implies that the density is really defined on a submanifold $\Cconstr^{d_X} := \{ \tilde{\x} = (\tilde{\x}_0, \dots, \tilde{\x}_{N-1}) \in \C^{d_X} \mid \tilde{\x}_{\kappa} = \tilde{\x}^*_{N - \kappa} \forall \kappa \in [N]\}$ of complex signals fulfilling the constraint. We can define a coordinate chart $\trunc : \Cconstr^{d_X} \rightarrow \R^{d_X}$ on this submanifold by extracting the unconstrained part of a DFT $\tilde{\x}$. This operator simply concatenates the relevant real and imaginary parts of the DFT and is defined as follows for all $\tilde{\x} \in \Cconstr^{d_X}$:
\begin{align} \label{eq:coord_chart}
  \trunc[\tilde{\x}] = \begin{cases}
     \left( \Re[\tilde{\x}_\kappa] \right)_{\kappa = 0}^{N/2} \oplus \left(\Im [\tilde{\x}_\kappa] \right)_{\kappa = 1}^{N/2 - 1}  & \text{ if } N \in 2 \N \\
     \left( \Re[\tilde{\x}_\kappa] \right)_{\kappa = 0}^{\lfloor N/2 \rfloor} \oplus \left(\Im [\tilde{\x}_\kappa] \right)_{\kappa = 1}^{\lfloor N/2 \rfloor}  & \text{ else,}
    \end{cases}
\end{align}
where $\boldv_1 \oplus \boldv_2$ denotes the concatenation of two vectors $\boldv_1 \in \R^{d_1}$ and $\boldv_2 \in \R^{d_2}$, with $d_1, d_2 \in \N$. Due to \cref{eq:real_constraint}, one can unambiguously reconstruct $\tilde{\x} \in \Cconstr^{d_X}$ from $\trunc[\tilde{\x}]$. Hence, the coordinate chart admits an inverse $\trunc^{-1}: \R^{d_X} \rightarrow \Cconstr^{d_X}$ defined as follows for all $\z = (\z_0, \dots, \z_{N-1}) \in \R^{d_X}$:
\begin{align} \label{eq:inv_coord_chart}
  \trunc^{-1}[\z] = \begin{cases}
     (\z_0) \oplus (\z_\kappa +  i \cdot \z_{N/2 + \kappa})_{\kappa=1}^{N/2 - 1} \oplus (\z_{N/2}) \oplus (\z_{N/2 - \kappa} -  i \cdot \z_{N - \kappa})_{\kappa=1}^{N/2 - 1} & \text{ if } N \in 2 \N \\
      (\z_0)  \oplus (\z_\kappa +  i \cdot \z_{\lfloor N/2 \rfloor + \kappa})_{\kappa=1}^{\lfloor N/2 \rfloor} \oplus (\z_{\lceil N/2 \rceil - \kappa} -  i \cdot \z_{N - \kappa})_{\kappa=1}^{\lfloor N/2 \rfloor} & \text{ else.}
    \end{cases}
\end{align}
With this coordinate chart, it becomes possible to rigorously define the density $\tilde{p} : \Cconstr^{d_X} \rightarrow \R^+$. One simply defines a probability density $\tilde{p}_{\trunc}: \R^{d_X} \rightarrow \R^+$ over the real vector space $\R^{d_X}$ on which the coordinate chart is defined and pull it back to the manifold of constrained signals $\tilde{p} := \tilde{p}_{\trunc} \circ \trunc$. This indeed defines a density that depends on the real and imaginary parts of the frequency representations $\tilde{\x} \in \Cconstr^{d_X}$ of time series, as announced in \cref{sec:background}.

Finally, it remains to rigorously define the score $\tilde{\s} : \Cconstr^{d_X} \times [0, T] \rightarrow \C^{d_X}$ in the frequency domain. This can be done by starting from the real score $\tilde{\s}_{\trunc}: \R^{d_X} \times [0, T] \rightarrow \R^{d_X}$ that is well-defined for all $\z \in \R^{d_X}$ and $t \in [0,T]$ as  $\tilde{\s}_{\trunc}(\z, t) = \nabla_{\z} \log \tilde{p}_{\trunc, t}(\z)$. Again, one can expand this vector field to the constrained manifold by defining for all $\tilde{\x} \in \Cconstr^{d_X}$ and all $t \in [0, T]$:

\begin{align} \label{eq:score_dft_def}
    \tilde{\s}(\tilde{\x}, t) := \trunc^{-1}[\tilde{\s}_{\trunc}(\trunc(\tilde{\x}), t)].
\end{align}

This indeed defines a vector field involving partial derivatives of the log density with respect to the real and imaginary parts of the frequency representations $\tilde{\x} \in \Cconstr^{d_X}$ of time series and respects the mirror symmetry by virtue of \cref{eq:inv_coord_chart}. Everything is then consistent with the discussion from \cref{sec:background}.

\subsection{Diffusion SDEs in the frequency domain} \label{subapp:theory_sdes}

\dftbrownianlemma*
\begin{proof}
\textbf{(1) Mirror Symmetry.} This point directly follows from the symmetry of the DFT, proved in \cref{prop:symmetry}.

 In what follows, we consider without loss of generality the case $M=1$, since the cases $M>1$ can be handled similarly by flattening matrices into vectors and using the same arguments as below. Let us first decompose $U$ into its real and imaginary parts, i.e. $U = U_{re} + i U_{im}$, where $U_{re}$ and $U_{im}$ are in $\mathbb{R}^{N \times N}$. Note that these two matrices are both symmetric. Computing the distribution of the components of $\boldv$ will require the knowledge of covariance matrices, which will depend on $U_{re}^2$, $U_{im}^2$ and $U_{re}U_{im}$. For any $\kappa$ and $\tau$ in $[N]$, 
\begin{align*}
    [U_{re}^{2}]_{\kappa, \tau} &= \frac{1}{N} \sum_{\gamma = 0}^{N-1} \cos(\frac{2\pi}{N}\kappa \gamma) \cos(\frac{2\pi}{N}\tau \gamma)  \\
    &= \frac{1}{2N} \sum_{\gamma = 0}^{N-1} \bigl( \cos(\frac{2\pi}{N}(\kappa+\tau)\gamma) + \cos(\frac{2\pi}{N}(\kappa -\tau)\gamma) \bigl)
\end{align*}

Similarly, \begin{align*}
    [U_{im}^{2}]_{\kappa, \tau} &= \frac{1}{N} \sum_{\gamma = 0}^{N-1} \sin(\frac{2\pi}{N}\kappa \gamma) \sin(\frac{2\pi}{N}\tau \gamma)  \\
    &= \frac{1}{2N} \sum_{\gamma = 0}^{N-1} \bigl( \cos(\frac{2\pi}{N}(\kappa-\tau)\gamma)  - \cos(\frac{2\pi}{N}(\kappa + \tau)\gamma) \bigl)
\end{align*}
To compute these sums, we consider the situations where:
\begin{itemize}
    \item $N \vert \kappa - \tau$: this is equivalent to \underline{$\kappa = \tau$}, since $-(N-1) \leq \kappa -\tau \leq N-1$
    \item $N \vert \kappa + \tau$: this is equivalent to \underline{$\kappa = \tau = 0$} or \underline{$\kappa + \tau = N$}, since $0 \leq \kappa + \tau \leq 2(N-1)$
\end{itemize}
Hence,  
if $N$ is even:
\begin{equation}
  [U_{re}^{2}]_{\kappa, \tau} =
    \begin{cases}
      1 & \text{if $\kappa = \tau = 0$ or $\kappa = \tau = \frac{N}{2}$} \\
      \frac{1}{2} & \text{if $\kappa \notin \{0, N/2\}$ and $\kappa = \tau$ or $\kappa = N- \tau$ }\\
      0 & \text{otherwise}
    \end{cases}  
    \text{and }
    [U_{im}^{2}]_{\kappa, \tau} =
    \begin{cases}
      0 & \text{if $\kappa = \tau = 0$ or $\kappa = \tau = \frac{N}{2}$} \\
      \frac{1}{2} & \text{if $\kappa \notin \{0, N/2\}$ and $\kappa = \tau$}\\
      -\frac{1}{2} & \text{if $\kappa \neq \ N/2$ and $\kappa = N - \tau$}\\
      0 & \text{otherwise}
    \end{cases}     
\end{equation}
If $N$ is odd: \begin{equation}
  [U_{re}^{2}]_{\kappa, \tau} =
    \begin{cases}
      1 & \text{if $\kappa = \tau = 0$} \\
      \frac{1}{2} & \text{if $\kappa \neq 0$ and $\kappa = \tau$, or $\kappa = N- \tau$ }\\
      0 & \text{otherwise}
    \end{cases}
    \text{and }
    [U_{im}^{2}]_{\kappa, \tau} =
    \begin{cases}
      0 & \text{if $\kappa = \tau = 0$} \\
      \frac{1}{2} & \text{if $\kappa \neq 0$ and $\kappa = \tau$}\\
      -\frac{1}{2} & \text{if $\kappa = N - \tau$}\\
      0 & \text{otherwise}
    \end{cases}    
\end{equation}

Finally, we compute $U_{re}U_{im}$:

\begin{align*}
    [U_{re}U_{im}]_{\kappa, \tau} &= \frac{1}{N} \sum_{\gamma = 0}^{N-1} \cos(\frac{2\pi}{N}\kappa \gamma) \sin(\frac{2\pi}{N}\tau \gamma)  \\
    &= \frac{1}{2N} \sum_{\gamma = 0}^{N-1} \bigl( \sin(\frac{2\pi}{N}(\tau-\kappa)\gamma)  + \sin(\frac{2\pi}{N}(\kappa + \tau)\gamma) \bigl) \\
    &= 0
\end{align*}
Hence, $U_{re}U_{im} = 0_{N}$ and similarly $U_{im}U_{re} = 0_{N}$ by taking the transpose and using the symmetry of $U_{re}$ and $U_{im}$. We can now characterize the distribution followed by the stochastic process $\boldv$. We first write $U_{col} = \begin{pmatrix} U_{re} \\ U_{im}\end{pmatrix} $, and then notice that $\boldv$ can be investigated through the lens of its flattened version $\boldv_{flat} = U_{col}\boldw$, which is a stochastic process in $\R^{2N}$. 

\textbf{(2) Real Brownian Motion.} First, ${\boldv}_{0}$ is real-valued, by using \textbf{(1) Mirror Symmetry}. We then have the following:
\begin{itemize}
    \item ${\boldv}_{0}(0) = 0$ almost surely: this stems from $\boldw(0) = 0$ almost surely, since $\boldw$ is a multivariate standard Brownian motion and $\boldv = U \boldw$.
    \item Continuity of $t \rightarrow {\boldv}_{0}(t)$ almost surely:  $\boldw$ satisfies the continuity property and the DFT operator (seen as a complex operator) is linear, hence ${\boldv}_{0}$ is also continuous with respect to $t$ almost surely.
    \item Stationary and independent increments: this follows from the linearity of the DFT operator and $\boldw$ being a Brownian motion.
    \item For any $t,s \geq 0$, ${\boldv}_{0}(t+s) - {\boldv}_{0}(s) \sim \mathcal{N}(0, t)$: to see this, we notice that ${\boldv}_{0}(t+s) - {\boldv}_{0}(s)$ is Gaussian\footnote{Note that a random variable almost surely equal to $0$ can be seen as a degenerate Gaussian with mean $0$ and variance $0$.} since it is a linear transform of  $\boldw(t+s) - \boldw(s)$. Moreover, its mean and its variance are given by:
    \begin{align*}
        \mathbb{E}\bigr[{\boldv}_{0}(t+s) - {\boldv}_{0}(s)\bigr] &=  \bigr[U_{col}\mathbb{E}[\boldw(t+s)-\boldw(s)]\bigr]_{0} \\ &= 0
    \end{align*}
    \begin{align*}
        \text{Var}({\boldv}_{0}(t+s) - {\boldv}_{0}(s)) &= \bigr[U_{col}\text{Cov}\bigl(\boldw(t+s)-\boldw(s), \boldw(t+s)-\boldw(s)\bigl)U_{col}^{T})\bigr]_{0,0} \\ &= t[U_{col} U_{col}^{T}]_{0,0} \\ &= t[U_{re}^2]_{0,0} \\ &= t
    \end{align*}
\end{itemize}
Hence, we have shown that $\boldv_{0}$ is a real Brownian motion.

\textbf{(3) Complex Brownian Motions.} Let $1 \leq \kappa < \lfloor N/2 \rfloor$. Then, $\Re(\boldv_{\kappa})$ and $\Im(\boldv_{\kappa})$ follow the first three properties of a Brownian motion, using the same arguments as above. For the last point, we first characterize the distribution of $\Re(\boldv_{\kappa})$ and $\Im(\boldv_{\kappa})$, and then show that they are independent.

    \emph{Distribution of $\Re(\boldv_{\kappa})$.} $\sqrt{2}\Re(\boldv_{\kappa})$ is a standard Brownian motion in $\mathbb{R}$:   For any $t,s \geq 0$, $\Re(\boldv_{\kappa})(t+s) - \Re(\boldv_{\kappa})(s)$  is  Gaussian since it is a linear transform of $\boldw(t+s)-\boldw(s)$ which is a Gaussian vector. 
    We can compute its mean and its variance: 
    \begin{align*}
        \mathbb{E}\bigr[\Re(\boldv_{\kappa})(t+s) - \Re(\boldv_{\kappa})(s)\bigr] &= \bigr[U_{col}\mathbb{E}[\boldw(t+s)-\boldw(s)]\bigr]_{\kappa} \\ &= 0
    \end{align*}
    \begin{align*}
        \text{Var}\bigl(\Re(\boldv_{\kappa})(t+s) - \Re(\boldv_{\kappa})(s)\bigl) &= \bigr[ U_{col}\text{Cov}\bigl(\boldw(t+s)-\boldw(s), \boldw(t+s)-\boldw(s)\bigl)U_{col}^{T}) \bigr]_{\kappa,\kappa} \\ &=  t[U_{col} U_{col}^{T}]_{\kappa,\kappa} \\ &= t[U_{re}^2]_{\kappa,\kappa} \\ & = \frac{1}{2}t
    \end{align*}
    
    \emph{Distribution of $\Im(\boldv_{\kappa})$.} Similarly, we can prove with the same arguments that $\sqrt{2}\Im(\boldv_{\kappa})$ is a standard Brownian motion in $\mathbb{R}$.

    \emph{Independence of $\Re(\boldv_{\kappa})$ and $\Im(\boldv_{\kappa})$.} Let $k$ and $m$ be two strictly positive integers, and let $(t_1, ..., t_{k}) \in {(\mathbb{R}^{*}_{+})}^{k}$ and $(t'_1, ..., t'_{m}) \in {(\mathbb{R}^{*}_{+})}^{m}$. 
    We need to show that the vectors $\boldv_{\kappa}^{re} = \bigl(\Re(\boldv_{\kappa})(t_1), ...,  \Re(\boldv_{\kappa})(t_{k})\bigl)$ and $\boldv_{\kappa}^{im} = \bigl(\Im(\boldv_{\kappa})(t'_1), ...,  \Im(\boldv_{\kappa})(t'_{m})\bigl)$ are independent. 
    First, the concatenation of $\boldv_{\kappa}^{re}$ and $\boldv_{\kappa}^{im}$ can be expressed as a linear transform of $(\boldw(t_1), ..., \boldw(t_{k}), \boldw(t'_1) ,..., \boldw(t'_m))$, which is a Gaussian vector since $\boldw$ is a Brownian motion. Consequently, $(\boldv_{\kappa}^{re}, \boldv_{\kappa}^{im})$ is also a Gaussian vector. Now, let $l \in \{1,..., k\}$ and $n \in \{1,...,m \}$. Then, 
    \begin{align*}
        \text{Cov} \bigl(\Re(\boldv_{\kappa})(t_{l}), \Im(\boldv_{\kappa})(t'_{n}) \bigl) &= \bigr[  U_{re} \text{Cov} \bigl(\boldw(t_l), \boldw(t'_n) \bigl) U_{im} \bigr]_{\kappa, \kappa} \\ &= \min(t_l,t'_n) [U_{re}U_{im}]_{\kappa, \kappa} \\ &= 0
    \end{align*}
    Given this covariance structure and the fact that $(\boldv_{\kappa}^{re}, \boldv_{\kappa}^{im})$ is a Gaussian vector, we have $\boldv_{\kappa}^{re} \indep \boldv_{\kappa}^{im}$. Since this holds true for any choice of $(t_1, ..., t_{k}) \in {\mathbb{R}^{*}_{+}}^{k}$ and $(t'_1, ..., t'_{m}) \in {\mathbb{R}^{*}_{+}}^{m}$, we conclude that $\Re(\boldv_{\kappa}) \indep \Im(\boldv_{\kappa})$. The case $k = \lfloor \frac{N}{2} \rfloor$ can be handled using the same arguments, by distinguishing the cases $N$ odd and $N$ even.

\textbf{(4) Independence.} The mutual independence of the stochastic processes $\{\tilde{\boldv}_\kappa\}_{\kappa=0}^{\lfloor N/2 \rfloor}$ follows from the structure of $U_{re}^2$ and $U_{im}^2$. Indeed, for any $m$ and $n$ such that $m \neq n$, $0 \leq m \leq \lfloor \frac{N}{2} \rfloor$, and $0 \leq n \leq \lfloor \frac{N}{2} \rfloor$, we have $[U_{re}^2]_{m,n} = [U_{im}^2]_{m,n} = [U_{re}U_{im}]_{m,n} = 0$. We can then apply the same argument as in 3) of  \textbf{(3) Complex Brownian Motions} to obtain the mutual independence of the stochastic processes $\{\tilde{\boldv}_\kappa\}_{\kappa=0}^{\lfloor N/2 \rfloor}$.
\end{proof}

\dftsde*
\begin{proof}

\emph{Forward SDE.} Since $\tilde{\x} = U \x$, we can apply the multivariate Itô's lemma (Eq. 8.3, \cite{kloeden1992stochastic}), and obtain a forward SDE for $\tilde{\x}$:
\begin{equation} \label{eq:forward_sde_fourier}
    \d \tilde{\x} = U\boldf(\x,t)\d t + g(t)U\d\boldw'
\end{equation}
where we have implicitly used the fact that $\boldG(t) = g(t)I_{N}$ and $U$ commute.
By \cref{lemma:brownian_transform}, $\tilde{\boldv} = U\boldw' $ is a mirrored Brownian motion on $\C^{d_X}$, which gives the result.

\emph{Reverse SDE.} 
In order to derive the reverse SDE for $\tilde{\x}$, we follow three steps: (1) we write a forward SDE for which the truncation $\trunc[\tilde{\x}]$ is a solution ; (2) we write its associated reverse-time SDE \cite{anderson1982} and (3) we derive the full reverse SDE for the stochastic process $\tilde{\x}$.

\underline{Step 1:}
From \cref{eq:forward_sde_fourier} and using \cref{lemma:brownian_transform}, we can extract the following forward SDE for $\trunc[\tilde{\x}]$, which is defined in \cref{eq:coord_chart} :
\begin{equation} \label{eq:forward_sde_truncation}
\d\trunc[\tilde{\x}] = \trunc \bigr[U\boldf(U^*\boldh(\trunc[\tilde{\x}]), t)\bigr]\d t+ g(t)\boldA \d \breve{\boldw}
\end{equation}
where:
\begin{itemize}
\item $\boldh: \R^{d_X} \rightarrow \C^{d_X}$ satisfies $\boldh(\trunc[U\y]) = U\y$ for all $\y \in \R^{d_X}$ as defined in \cref{eq:inv_coord_chart}.
    \item $\boldA \in \R^{N \times N}$ is a diagonal matrix such that $[\boldA]_{\kappa, \kappa} = \begin{cases} 1 & \text{if $\kappa = 0$, or $N$ is even and $\kappa = N/2$} \\
                \frac{1}{\sqrt{2}} & \text{otherwise}
\end{cases}$, 
\item $\breve{\boldw}$ is a stochastic process in $\R^{d_X}$ which satisfies $\boldA\breve{\boldw} = \trunc[U\boldw']$. The above point, together with \cref{lemma:brownian_transform}), then implies that $\breve{\boldw}$ is in fact a standard multivariate Brownian motion.
\end{itemize}   

\underline{Step 2:}
The associated reverse-time SDE \cite{anderson1982} is given by:
\begin{equation}
    \d\trunc[\tilde{\x}] = \bigl\{ \trunc \bigr[ U\boldf(U^*\boldh(\trunc[\tilde{\x}]),t) \bigr] - g(t)^2 \boldA \boldA^{T} \nabla_{\trunc[\tilde{\x}]} \log \tilde{p}_{t}(\trunc[\tilde{\x}]) \bigl\} \d t + g(t)\boldA \d \widehat{\boldw}
\end{equation}
where $\widehat{\boldw}$ is a standard Brownian motion on $\R^{d_X}$.

\underline{Step 3:} Since $\boldh(\trunc[\tilde{\x}]) = \tilde{\x}$, we can recover the reverse-time SDE followed by $\tilde{\x}$ by applying the operator $\boldh$ to $\trunc[\tilde{\x}]$, and using the Itô's lemma:

\begin{equation} \label{eq:reverse_fourier_unsimplified}
    \d \tilde{\x} = \bigl\{ \boldh(\trunc \bigr[ U\boldf(U^*\boldh(\trunc[\tilde{\x}]), t) \bigr]) - g(t)^2 \boldh \bigl(\boldA^2 \nabla_{\trunc[\tilde{\x}]} \log \tilde{p}_{t}(\trunc[\tilde{\x}]) \bigl)\bigl\}\d t + g(t)\boldh(\boldA)\d\widehat{\boldw}
\end{equation}
where we have exploited the fact that $\boldh$ is linear, and with the slight abuse of notation $\boldh(\boldA) \in \C^{N \times N}$, which is the matrix obtained by applying $\boldh$ to the columns of $\boldA$. Using the definition of $\boldh$, \cref{eq:reverse_fourier_unsimplified} can further be simplified into:
\begin{equation} \label{eq:reverse_fourier_before_commute}
\d \tilde{\x} =  \bigl\{ U\boldf(U^*\tilde{\x}, t) - g(t)^2 \boldh \bigl(\boldA^2 \nabla_{\trunc[\tilde{\x}]} \log \tilde{p}_{t}(\trunc[\tilde{\x}]) \bigl)\bigl\}\d t + g(t)\d \breve{\boldv}
\end{equation}

where $\breve{\boldv} = \boldh(\boldA)\widehat{\boldw} = \boldh(\boldA \widehat{\boldw})$ is a mirrored Brownian motion. Finally, notice that for any $\y \in \R^{d_X}$, we have $\trunc^{-1}(\boldA^2\y) = \boldA^{2}\trunc^{-1}(\y)$, which follows from the definition of $\trunc^{-1}$.
 Hence, \cref{eq:reverse_fourier_before_commute} is equivalent to:
 \begin{equation} \label{eq:reverse_fourier_after_commute}
\d \tilde{\x} =  \bigl\{ U\boldf(U^*\tilde{\x}, t) - g(t)^2 \boldA^2 \tilde{\s}(\tilde{\x}, t) \bigl\}\d t + g(t)\d \breve{\boldv}
\end{equation}
where $\tilde{\s}(\tilde{\x}, t)$ has been defined in \cref{eq:score_dft_def}.

\end{proof}

\subsection{Denoising score matching in the frequency domain} \label{subapp:theory_score}

\dftscore*
\begin{proof}
Let $\boldA \in \R^{N \times N}$ be the diagonal matrix such that $[\boldA]_{\kappa, \kappa} = \begin{cases} 1 & \text{if $\kappa = 0$, or $N$ is even and $\kappa = N/2$} \\
                \frac{1}{\sqrt{2}} & \text{otherwise}
\end{cases}$ 

\underline{Step 1:} We first express the score of $\x$ with respect to the score of the truncation $\trunc[\tilde{\x}]$.

By definition of $\boldh$ in \cref{eq:inv_coord_chart}, we have $\x = U^*\boldh(\trunc[\tilde{\x}])$. Hence, we can write, using the change of variable formula: 
\begin{equation}
    p_{t|0}(\x(t) \vert \x(0)) = C \cdot \tilde{p}_{t|0} \bigl(\trunc[\tilde{\x}(t)] \vert \trunc[\tilde{\x}(0)]\bigl)
\end{equation}
where $C$ is a constant which does not depend on $\x$, since $\x \mapsto \trunc[U\x]$ is linear. Moreover, let us write $\x \mapsto \trunc[U\x]$ in matrix form, i.e. $\forall \x \in \mathbb{R}^{d_X}$, $\trunc[U\x] = VU_{col}\x = Q \x$, where $V \in \R^{N \times 2N}$, $U_{col} = \begin{pmatrix} U_{re} \\ U_{im}\end{pmatrix}$ and $Q$ is an invertible matrix in $\R^{N \times N}$. For the rest of the proof, we shall build on the below results:

\textcolor{Peach}{Result 1.} $QQ^{T} = \boldA^{2}$. To see this, write $QQ^{T} = VU_{col}U_{col}^{T}V^{T}$. The matrix $U_{col}U_{col}^{T}$ is equal to $\begin{pmatrix}
    U_{re}^2 & 0_{N} \\
    0_{N} & U_{im}^2 
\end{pmatrix}$ (cf. the proof to \cref{lemma:brownian_transform}), while the multiplication by $V$, on the left and on the right of $U_{col}U_{col}^{T}$, extracts the submatrix corresponding to the indices represented by the truncature $\trunc$. Hence, $VU_{col}U_{col}^{T}V^{T} = \boldA^{2}$.

\textcolor{Violet}{Result 2.}  For any $\x \in \R^{d_X}$ , we have $Q^{T}\x = U^*\trunc^{-1}[\boldA^2 \x]$. To see this, notice that \textcolor{Peach}{Result 1} implies that $Q^{T}\x = Q^{-1} \boldA^{2} \x = U^*\trunc^{-1}[\boldA^{2} \x]$ for all $\x \in \R^{d_X}$.

Equipped with these results, we can now complete the rest of the proof. First, we have:
\begin{align}
    \nabla_{\x(t)} \log p_{t|0}(\x(t) \vert \x(0)) &= \nabla_{\x(t)} \log \tilde{p}_{t|0}\bigl(\trunc[\tilde{\x}(t)] \vert \trunc[\tilde{\x}(0)]\bigl) \\
    &= Q^{T} \nabla_{\trunc[\tilde{\x}(t)]} \log \tilde{p}_{t|0}(\trunc[\tilde{\x}(t)] \vert \trunc[\tilde{\x}(0)]) \label{eq:change_variable} & \text{(Chain rule)}
\end{align}

\underline{Step 2:} We then obtain:
\begin{align*}
\scorematching\left(\s'_{\tilde{\theta}}, \s_{t \vert 0}, \x, t\right) &\coloneqq \|\s'_{\tilde{\theta}}(\x, t) -  \nabla_{\x(t)} \log p_{t|0}(\x(t) \vert \x(0))\|^2 \\ &= \|U\s'_{\tilde{\theta}}(\x, t) -  U\nabla_{\x(t)} \log p_{t|0}(\x(t) \vert \x(0))\|^2 & \text{(Parseval identity)}\\
&= \|\tilde{\s}_{\tilde{\theta}}(\tilde{\x}, t)  - UQ^{T} \nabla_{\trunc[\tilde{\x}(t)]} \log \tilde{p}_{t|0}(\trunc[\tilde{\x}(t)] \vert \trunc[\tilde{\x}(0)])\|^2 & \text{(\cref{eq:change_variable})}\\
&= \|\tilde{\s}_{\tilde{\theta}}(\tilde{\x}, t)  - UU^*\trunc^{-1}[\boldA^2\nabla_{\trunc[\tilde{\x}(t)]} \log \tilde{p}_{t|0}(\trunc[\tilde{\x}(t)] \vert \trunc[\tilde{\x}(0)]]\|^2 & \text{(\textcolor{Violet}{Result 2})}\\
&= \|\tilde{\s}_{\tilde{\theta}}(\tilde{\x}, t)  - \boldA^{2}\trunc^{-1}[\nabla_{\trunc[\tilde{\x}(t)]} \log \tilde{p}_{t|0}(\trunc[\tilde{\x}(t)] \vert \trunc[\tilde{\x}(0)]]\|^2  & \text{(\cref{prop:unitary_U} \& Definition of $\trunc^{-1}$)} \\
&= \|\tilde{\s}_{\tilde{\theta}}(\tilde{\x}, t)  - \boldA^{2}\tilde{\s}_{t \vert 0}(\tilde{\x}, t)\|^2 & \text{(\cref{eq:score_dft_def})} \\
&=
\scorematching\left(\tilde{\s}_{\tilde{\theta}}, \Lambda^{2} \tilde{\s}_{t \vert 0}, \tilde{\x}, t\right).
\end{align*}
\end{proof}

\section{Empirical details} \label{app:empirical_details}
\textbf{Compute resources.} All the models were trained and used for sampling on a single machine equipped with a 18-Core Intel Core i9-10980XE CPU, a NVIDIA RTX A4000 GPU and a NVIDIA GeForce RTX 3080.

\subsection{Details on datasets} \label{subapp:datasets_details}
In this subsection, we give detailed information about the $6$ datasets used throughout our experiments and the preprocessing steps for each of them.

\textbf{ECG.} We use two collections of heartbeat signals, from the MIT-BIH Arrhythmia Dataset and the PTB Diagnostic ECG Database \cite{kachuee2018ecg}.
No preprocessing was performed on this dataset.

\textbf{MIMIC-III.} MIMIC-III \cite{johnson2016mimic} is a database consisting of deidentified records for patients who were in critical unit care units. \textit{Preprocessing.} We use the "vitals labs" table of the database, which corresponds to time-varying vitals and labs.
We extract the rows of the dataset which correspond to the first 24 hours of stay by using \emph{MIMIC-Extract}~\cite{wang2020mimic}. The features are then standardized across all times and patients. We also perform imputation to handle missing values in the dataset. To do so, we consider the mean features (average measurement over 1 hour). For each patient, and missing value, we propagate the last observation forward if this is possible. If not, we fill the missing value with the mean value for the patient (which is computed over the whole stay). If no mean value is available, we fill the entry with 0. 

\textbf{NASDAQ-2019.} This dataset \cite{onyshchak2020} contains daily prices for tickers trading on NASDAQ, and contains prices for up to 1st of April 2020. \textit{Preprocessing.} We considered one year of daily prices from 1st of January 2019 to 1st of January 2020. Each sample corresponds to one stock, and we remove the stocks which are not active in this whole time interval, or contain missing values.

\textbf{NASA battery.} The NASA battery dataset \cite{saha2007} consists of profiles for Li-on batteries, under charge and discharge. \textit{Preprocessing.} For both the charge and discharge datasets, we bin the time values (bins of size $10$ for Charge, $15$ for Discharge) and compute the mean of each feature inside each bin.

\textbf{US-Droughts.} This dataset \cite{minixhofer2021} consists of drought levels in different US counties, from 2000 to 2020. \textit{Preprocessing.} We consider one year of history, from 1st of January 2011 to 1st of January 2012, and drop the columns with missing values.

\subsection{Details on evaluation} \label{subapp:eval_details}

\textbf{Sliced Wasserstein distances.}
The sliced Wasserstein distance \cite{bonneel2015sliced} is a metric which can handle high-dimensional distributions. It is motivated by the fact that the Wasserstein distance is easy to compute when comparing two one-dimensional distributions. The idea of the sliced Wasserstein distance is to map the high-dimensional distributions of interest to one-dimensional distributions, by considering random projections on vectors of the unit sphere. For two distributions $\mu_{1}$ and $\mu_{2}$, it can be written as: 

\begin{equation}
    SW_{p}(\mu_1, \mu_2) \coloneqq \int_{\mathbb{S}^{d-1}} W_{p}(P_{u}\#\mu_1, P_{u}\#\mu_2)  du 
\end{equation}
where $\mathbb{S}^{d-1}$ is the unit sphere in dimension $d$, $P_{u}(x) = u \cdot x$ denotes the projection of $x$ on $u$, $P_{u} \# \mu$ is the push-forward of $\mu$ by $P_{u}$, and $W_{p}$ is the Wasserstein distance of order $p$. To estimate this quantity in practice, we sample $n=10,000$ random vectors $\{u_i \vert i\in [n] \}$ which follow a uniform distribution in $\mathbb{S}^{d-1}$ and consider $p = 2$. Hence, we can approximate ${SW}_{p}$ by the Monte-carlo estimator:
\begin{equation}
    \hat{SW}_{p}(\mu_1, \mu_2)  = \frac{1}{n}\sum_{i=1}^{n} W_{p}(P_{u_{i}}\#\mu_1, P_{u_{i}}\#\mu_2)
\end{equation}

\textbf{Marginal Wasserstein distances.}
In addition to the sliced Wasserstein distance, we also consider the marginal Wasserstein distance. For any $j \in \{1,...,d\}$, the $j$-th marginal Wasserstein distance is defined as:
\begin{equation}
    {MW_{p}}^{(j)}(\mu_{1},\mu_{2}) = W_p(P_{e_j}\#\mu_{1}, P_{e_j} \# \mu_{2})  
\end{equation}
where $e_j$ is the $j$-th vector of the standard basis of $\R^{d}$. Throughout our experiments in \cref{sec:empirical}, we compute the Wasserstein distances with respect to $\D{train}$ and $\tildeD{train}$.

\subsection{Additional plots} \label{subapp:more_plots}

\textbf{Sliced Wasserstein distances.} In \cref{fig:sliced_box_plots}, we show the distribution of the sliced Wasserstein distances over all slices. In addition, we have included the average sliced Wasserstein distances obtained with 2 baselines. The first baseline is simply the Wasserstein distance between the training set and a set of sample only containing identical copies the average sample $SW(\D{train}, \S{mean})$, where $\S{mean} = \{ \mathbb{E}_{X \sim U(\D{train}}) [X]\} $. It represents the performance of a dummy generator that only generates the average time series and is denoted by \emph{mean} in \cref{fig:sliced_box_plots}. The second baseline is the Wasserstein distance between two random splits of the training set $SW(\D{train}^{1/2}, \D{train}^{2/2})$, where $\D{train} = \D{train}^{1/2} \bigsqcup \D{train}^{2/2}$ is a decomposition of the training set into two disjoint random splits of equal size $|\D{train}^{1/2}| = |\D{train}^{2/2}|$. It represents the distance between two samples from the ground-truth distribution and is denoted by \emph{self} in \cref{fig:sliced_box_plots}. As we can observe, both the time and frequency diffusion models substantially outperform the \emph{mean} baseline (as expected) and perform on par with the \emph{self} baseline. This indicates that the models learned a good approximation of the real distribution. Furthermore, we notice that the frequency diffusion models tend to have smaller quantiles than the time diffusion models. This confirms that frequency diffusion models outperform the time diffusion models, as discussed in \cref{sec:empirical}.

\textbf{Marginal Wasserstein distances.} In \cref{fig:marginal_box_plots}, we show the distribution of the marginal Wasserstein distances over all slices. In addition, we have included the average marginal Wasserstein distances obtained with the 2 baselines defined in the previous paragraph.  Again, both the time and frequency diffusion models tend to outperform the \emph{mean} baseline (as expected) and perform on par with the \emph{self} baseline. Furthermore, we notice that the frequency diffusion models tend to have smaller quantiles than the time diffusion models. This is consistent with the observations made in the above paragraph.

\textbf{Per-sample localization.} In \cref{fig:joind_deloc}, we observe the distribution of our localization metrics $\deloc{time}$ and $\deloc{sigma}$ for each sample and each dataset from \cref{sec:empirical}. We notice that most samples are located below the $y=x$ axis, which confirms the fact that most samples are more localized in the frequency domain. Interestingly, we also observe that none of the samples is located close to the origin. This confirms the uncertainty theorem from \cite{nam2013}.

\begin{figure}
\centering
    \begin{subfigure}{0.45\textwidth}
        \includegraphics[width=\textwidth]{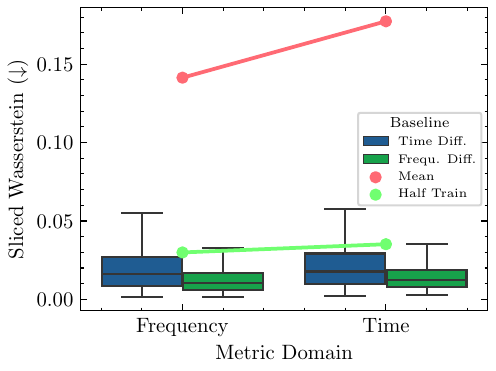}
        \caption{ECG.}
    \end{subfigure}
    \hfill
    \begin{subfigure}{0.45\textwidth}
        \includegraphics[width=\textwidth]{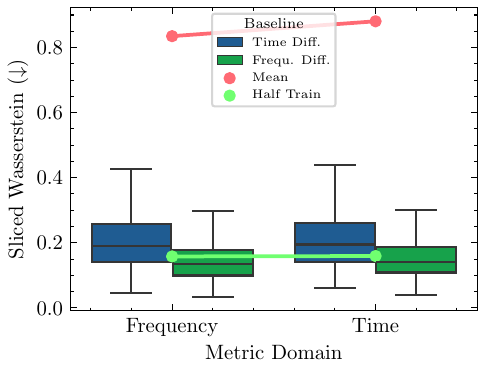}
        \caption{MIMIC-III.}
    \end{subfigure}
    \begin{subfigure}{0.45\textwidth}
        \includegraphics[width=\textwidth]{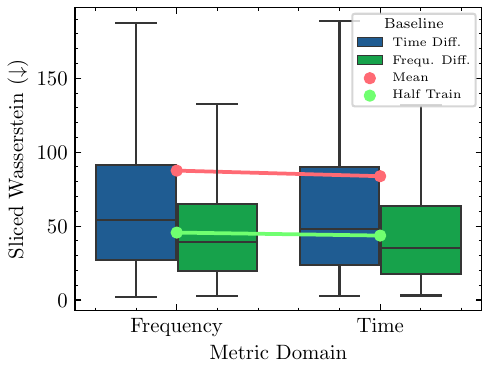}
        \caption{NASDAQ-2019.}
    \end{subfigure}
    \hfill
    \begin{subfigure}{0.45\textwidth}
        \includegraphics[width=\textwidth]{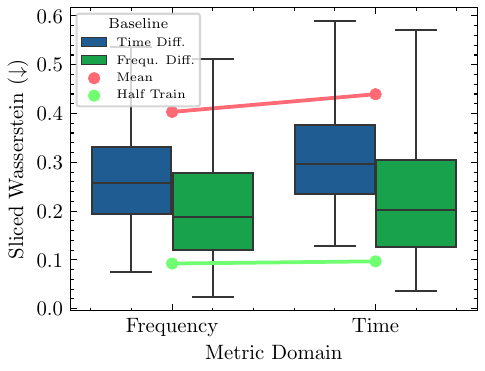}
        \caption{NASA-Charge.}
    \end{subfigure}
    \begin{subfigure}{0.45\textwidth}
        \includegraphics[width=\textwidth]{figures/sliced_wasserstein_nasdaq-2019.pdf}
        \caption{NASA-Discharge.}
    \end{subfigure}
    \hfill
    \begin{subfigure}{0.45\textwidth}
        \includegraphics[width=\textwidth]{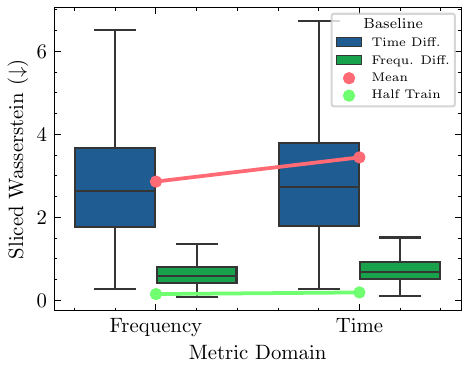}
        \caption{US-Droughts.}
    \end{subfigure}
    
    \caption{Sliced Wasserstein distances of time and frequency diffusion models.}
    \label{fig:sliced_box_plots} 
\end{figure}

\begin{figure}
\centering
    \begin{subfigure}{0.45\textwidth}
        \includegraphics[width=\textwidth]{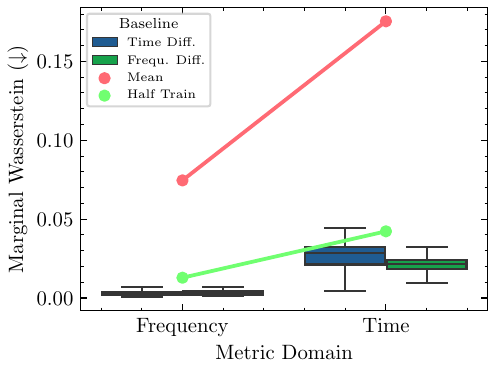}
        \caption{ECG.}
    \end{subfigure}
    \hfill
    \begin{subfigure}{0.45\textwidth}
        \includegraphics[width=\textwidth]{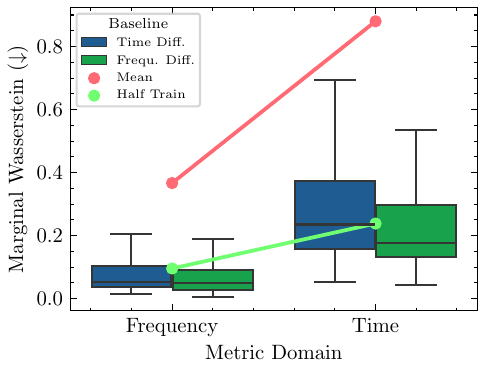}
        \caption{MIMIC-III.}
    \end{subfigure}
    \begin{subfigure}{0.45\textwidth}
        \includegraphics[width=\textwidth]{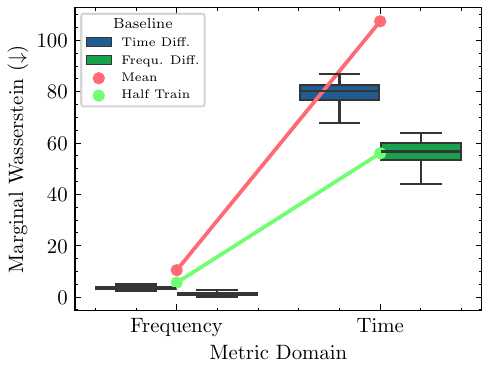}
        \caption{NASDAQ-2019.}
    \end{subfigure}
    \hfill
    \begin{subfigure}{0.45\textwidth}
        \includegraphics[width=\textwidth]{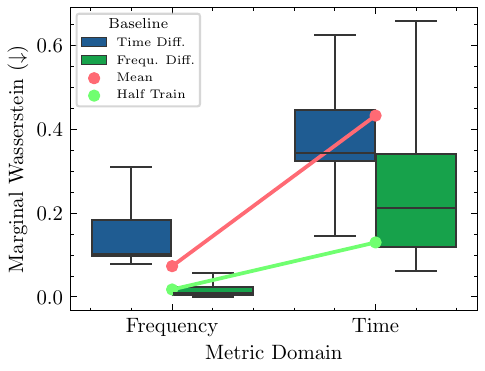}
        \caption{NASA-Charge.}
    \end{subfigure}
    \begin{subfigure}{0.45\textwidth}
        \includegraphics[width=\textwidth]{figures/marginal_wasserstein_nasdaq-2019.pdf}
        \caption{NASA-Discharge.}
    \end{subfigure}
    \hfill
    \begin{subfigure}{0.45\textwidth}
        \includegraphics[width=\textwidth]{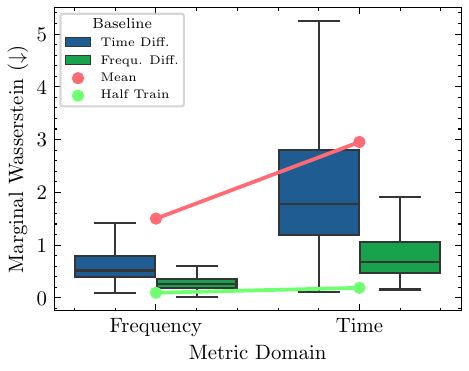}
        \caption{US-Droughts.}
    \end{subfigure}
    
    \caption{Marginal Wasserstein distances of time and frequency diffusion models.}
    \label{fig:marginal_box_plots} 
\end{figure}

\begin{figure}
\centering
    \includegraphics[width=.5\textwidth]{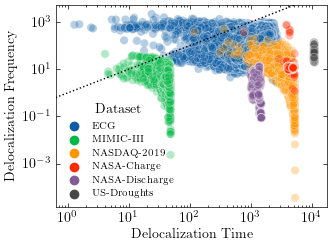}
    \caption{Localization metrics $\deloc{time}$ and $\deloc{freq}$ for all the samples of all datasets. We observe that no sample has a high localization (i.e. low $\Delta$) in the time and frequency domain simultaneously.}
    \label{fig:joind_deloc}
\end{figure}

\section{Alternative backbone} \label{app:other_backbone}
\begin{figure}
\centering
    \begin{subfigure}{0.45\textwidth}
        \includegraphics[width=\textwidth]{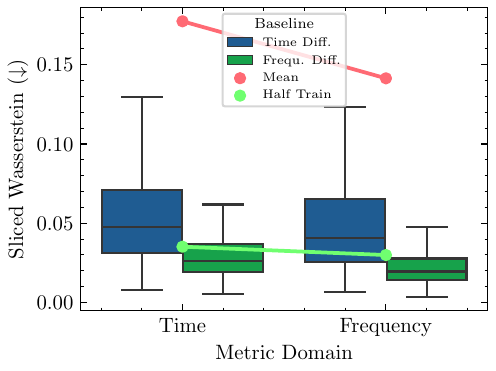}
        \caption{ECG.}
    \end{subfigure}
    \hfill
    \begin{subfigure}{0.45\textwidth}
        \includegraphics[width=\textwidth]{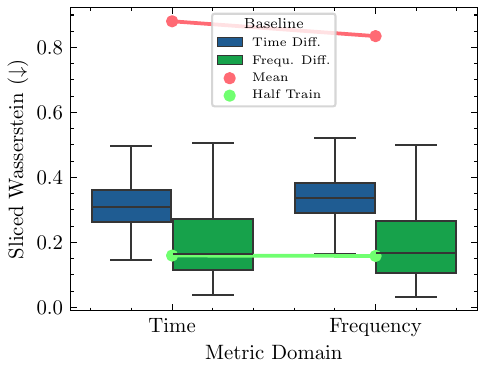}
        \caption{MIMIC-III.}
    \end{subfigure}
    \begin{subfigure}{0.45\textwidth}
        \includegraphics[width=\textwidth]{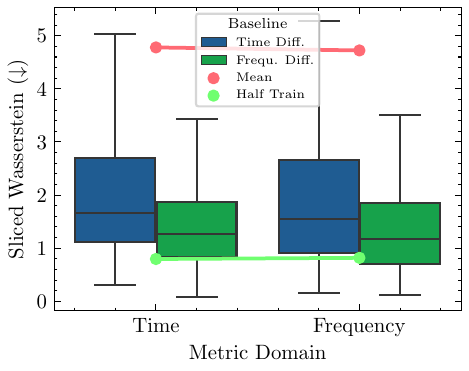}
        \caption{NASA-Discharge.}
    \end{subfigure}
    \hfill
    \begin{subfigure}{0.45\textwidth}
        \includegraphics[width=\textwidth]{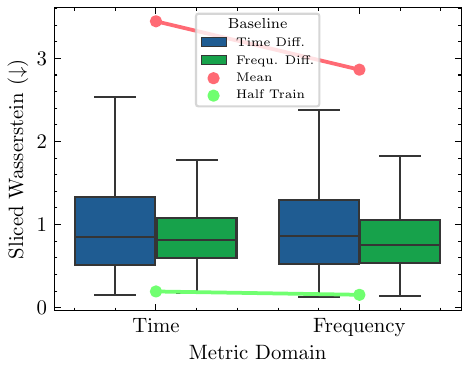}
        \caption{US-Droughts.}
    \end{subfigure}
    
    \caption{Sliced Wasserstein distances of time and frequency LSTM models.}
    \label{fig:sliced_box_plots_lstm} 
\end{figure}

\textbf{LSTM models.} For each dataset, we try an alternative parametrization of  the time score model $\s_{\theta}$ and the frequency score model $\tilde{\s}_{\tilde{\theta}}$ as LSTM encoders with 10 layers, each with dimension $d_{\mathrm{model}} = 72$. Both models have diffusion time $t$ encoding through random Fourier features composed with a learnable dense layer. This results in models with $427$k parameters. The data is noised by using a VP-SDE, as in \cite{song2020score}. The score models are trained with the denoising score-matching loss, as defined in \cref{sec:theory}. All the models are trained for 200 epochs with batch size $64$, AdamW optimizer and cosine learning rate scheduling ($20$ warmup epochs, $\mathrm{lr_{max}} = 10^{-3}$). The selected model achieves the lowest validation loss.

\textbf{Sliced Wasserstein distances.} In \cref{fig:sliced_box_plots_lstm}, we show the distribution of the sliced Wasserstein distances over all slices for the LSTM models. In addition, we have included the average sliced Wasserstein distances obtained with 2 baselines defined in \cref{subapp:more_plots}.  As observed for the transformer models, both the time and frequency diffusion models substantially outperform the \emph{mean} baseline (as expected) and perform on par with the \emph{self} baseline. This indicates that the models learned a good approximation of the real distribution. Furthermore, we notice that the frequency diffusion models tend to have smaller quantiles than the time diffusion models. This confirms that frequency diffusion models outperform the time diffusion models, as observed for the transformer models in \cref{sec:empirical}. We note that the Nasa-Charge and the NASDAQ-2019 are absent from \cref{fig:sliced_box_plots_lstm}. This is because we did not manage to obtain diffusion models performing better than the \emph{mean} baseline for these datasets, hence leading to non informative comparisons between models with poor performances.

\textbf{Other attempts.} In order to minimize the inductive bias in our models, we also tried to train diffusion models with simple feed-forward neural networks. Unfortunately, this attempt was unsuccessful and resulted in models performing worse than the \emph{mean} baseline in each case. This emphasizes the value of incorporating inductive biases in time series diffusion models.   

\section{Sample visualization} \label{app:sample_viz}
In \cref{fig:ECG_samples,fig:NASA_charge_samples,fig:NASA_discharge_samples,fig:NASDAQ_samples,fig:droughts_samples}, we visualize a few examples generated by each diffusion model, along with ground-truth training examples. We do not include samples from the MIMIC-III dataset in accordance with the dataset licence. We observe that the frequency diffusion models generate samples that are substantially less noisy than than the ones generated by time diffusion models. All the generated samples resemble training samples, with the only exception of the NASDAQ-2019 dataset. The models appear to be struggling with the high correlation between the different features.

\begin{figure}
    \centering
    \includegraphics[scale=0.3]{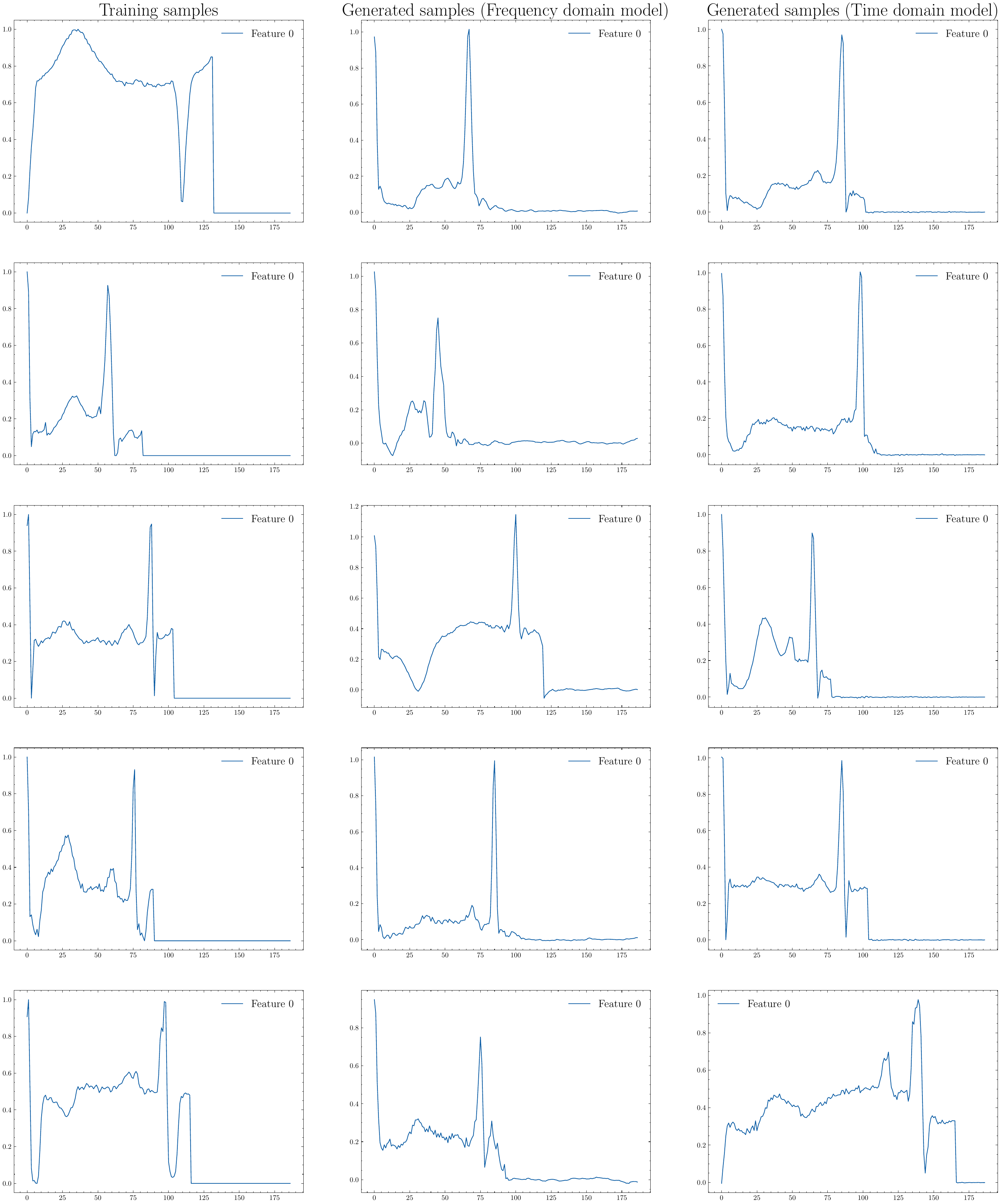}
    \caption{Samples for the ECG dataset. The $y$ axis corresponds to the different features, while the $x$ axis corresponds to the different time steps.}
    \label{fig:ECG_samples}
\end{figure}

\begin{figure}
    \centering
    \includegraphics[scale=0.3]{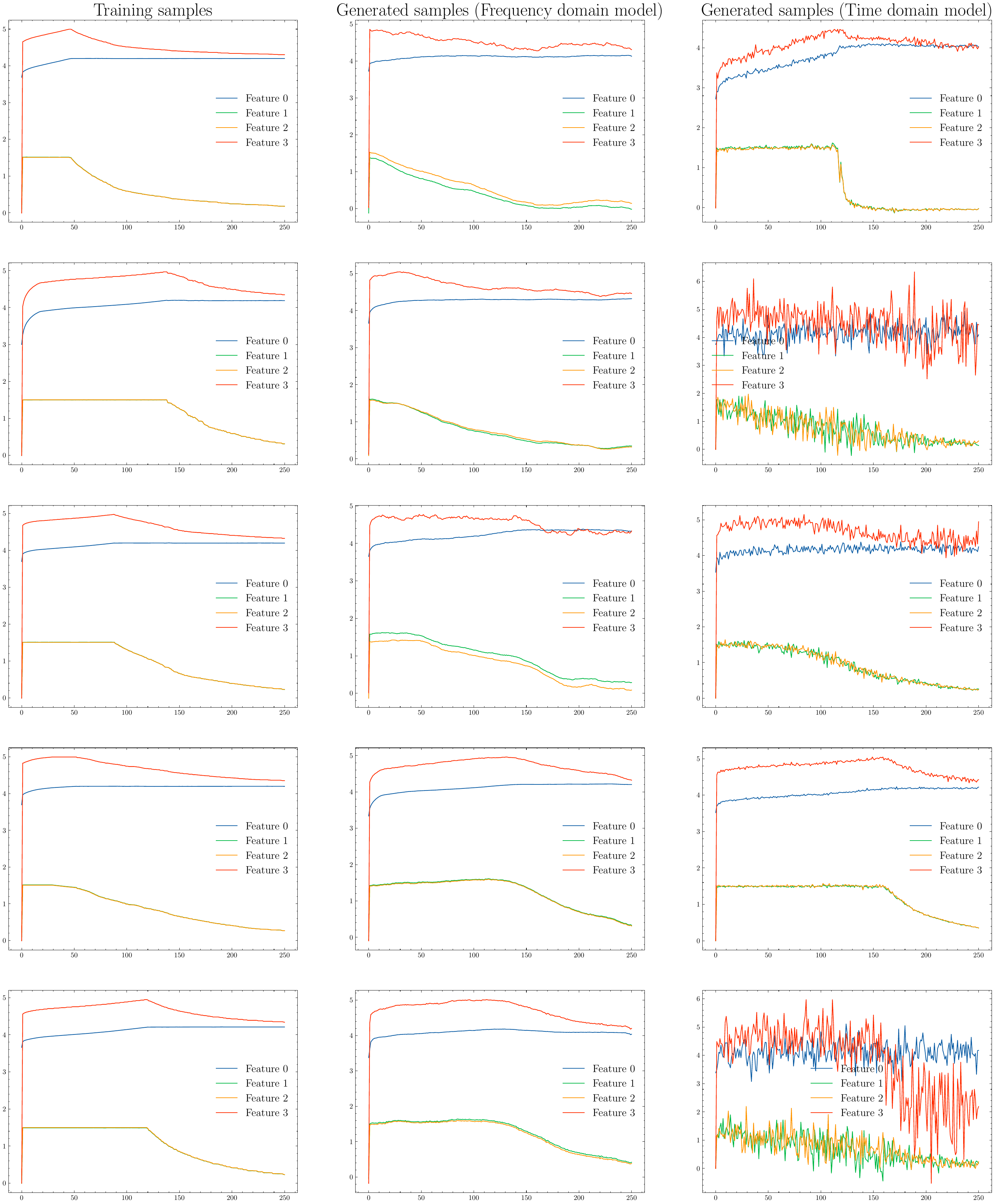}
    \caption{Samples for the NASA-Charge dataset. The $y$ axis corresponds to the different features, while the $x$ axis corresponds to the different time steps.}
    \label{fig:NASA_charge_samples}
\end{figure}

\begin{figure}
    \centering
    \includegraphics[scale=0.3]{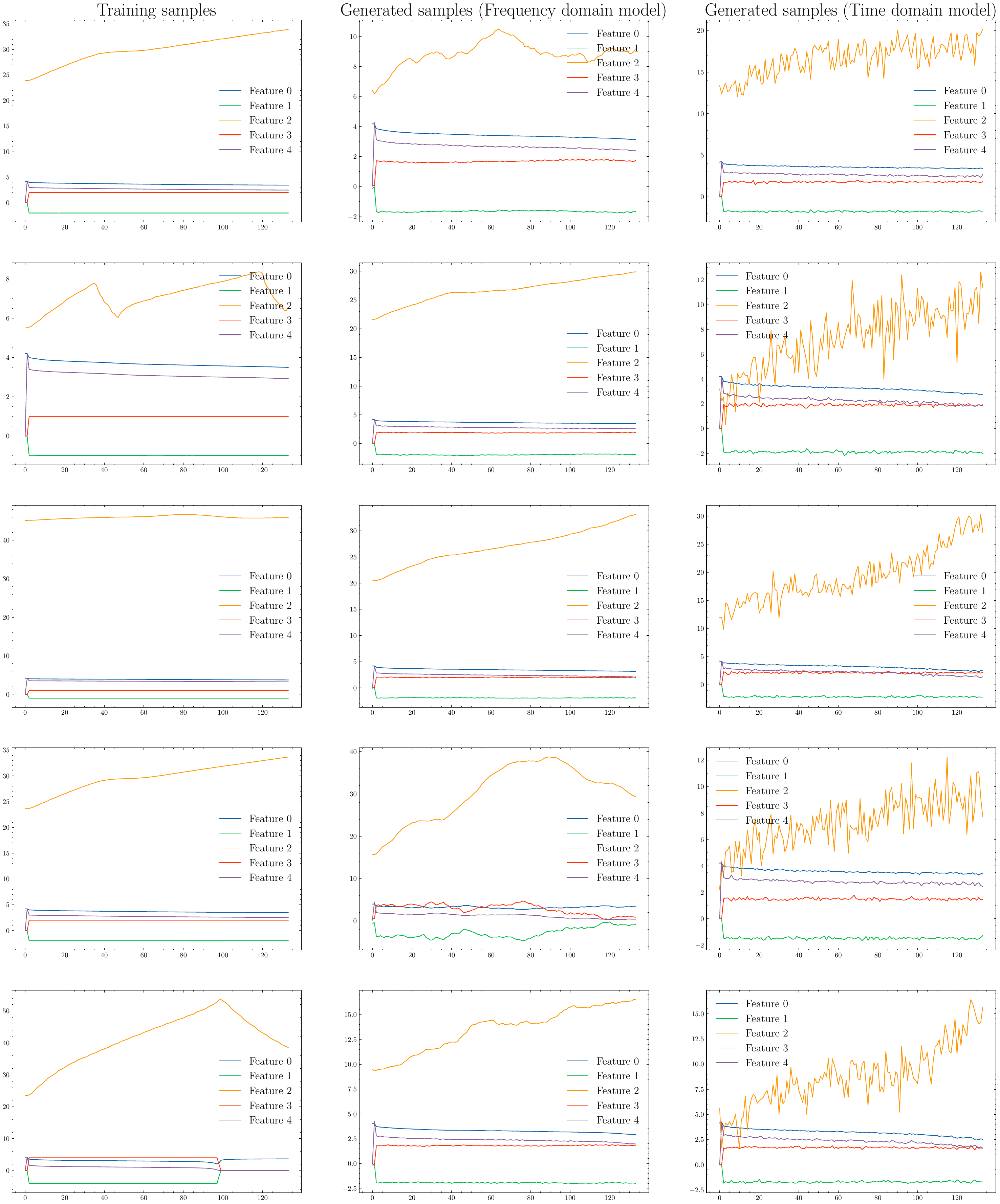}
    \caption{Samples for the NASA-Discharge dataset. The $y$ axis corresponds to the different features, while the $x$ axis corresponds to the different time steps.}
    \label{fig:NASA_discharge_samples}
\end{figure}

\begin{figure}
    \centering
    \includegraphics[scale=0.3]{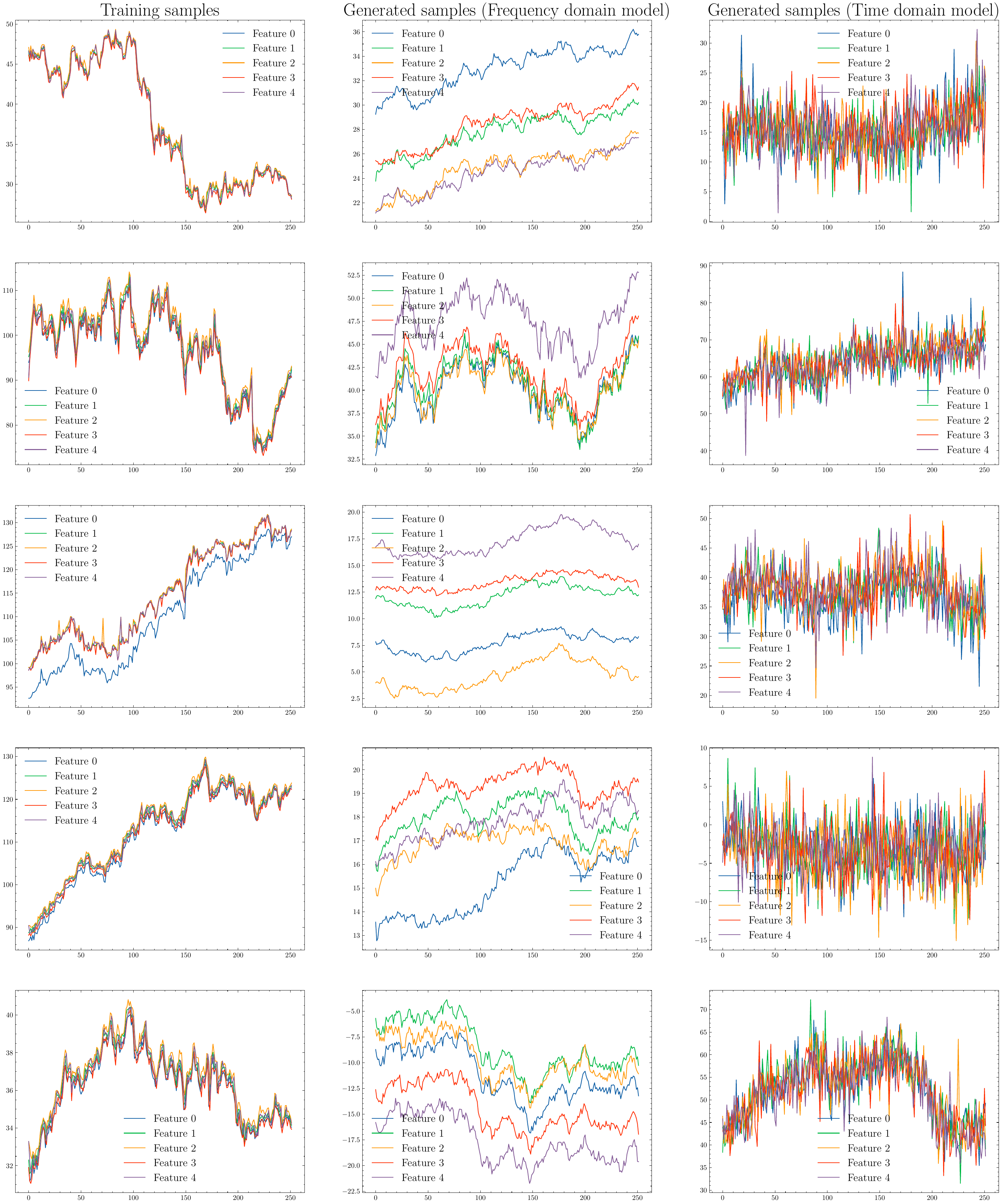}
    \caption{Samples for the NASDAQ-2019 dataset. The $y$ axis corresponds to the different features, while the $x$ axis corresponds to the different time steps.}
    \label{fig:NASDAQ_samples}
\end{figure}

\begin{figure}
    \centering
    \includegraphics[scale=0.3]{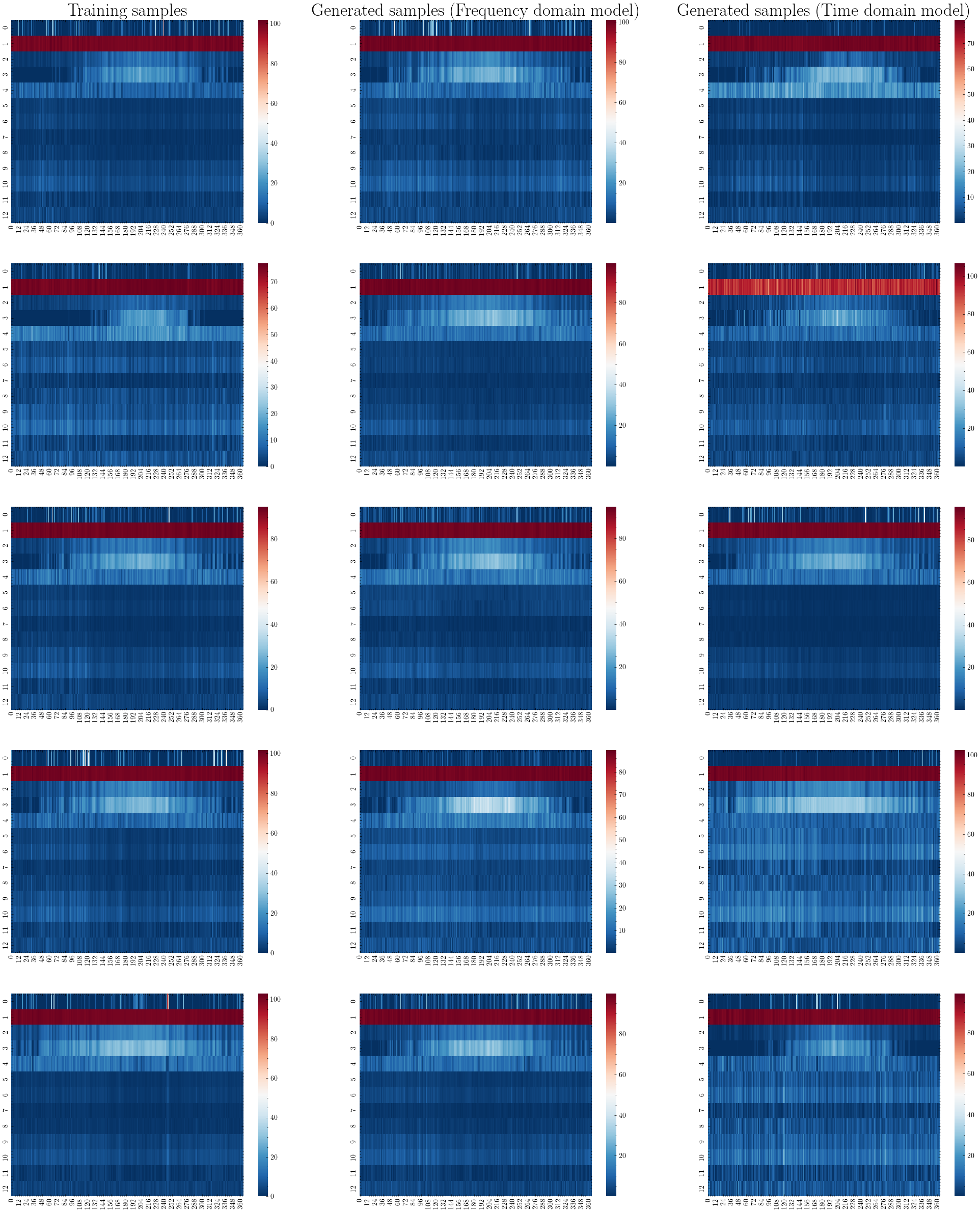}
    \caption{Samples for the US-Droughts dataset, represented as heatmaps. The $y$ axis corresponds to the different features, while the $x$ axis corresponds to the different time steps.}
    \label{fig:droughts_samples}
\end{figure}

\end{document}